\documentclass[twoside]{article}
\usepackage[accepted]{aistats2017}
\usepackage{algorithm}
\usepackage[noend]{algorithmic}
\usepackage{amsmath}
\usepackage{amssymb}
\usepackage{amsthm}
\usepackage{bbm}
\usepackage{bm}
\usepackage{color}
\usepackage{dsfont}
\usepackage{enumerate}
\usepackage{graphicx}
\usepackage{mathtools}
\usepackage{subfigure}
\usepackage{times}

\usepackage[usenames,dvipsnames]{xcolor}
\usepackage[bookmarks=false]{hyperref}
\hypersetup{
  pdftex,
  pdffitwindow=true,
  pdfstartview={FitH},
  pdfnewwindow=true,
  colorlinks,
  linktocpage=true,
  linkcolor=Green,
  urlcolor=Green,
  citecolor=Green
}

\usepackage[capitalize,noabbrev]{cleveref}

\newcommand{\commentout}[1]{}
\newcommand{\junk}[1]{}
\newcommand{\etal}{\emph{et al.}}

\newtheorem{theorem}{Theorem}
\newtheorem{corollary}{Corollary}
\newtheorem{lemma}{Lemma}

\newcommand{\cE}{\mathcal{E}}
\newcommand{\ccE}{\overline{\cE}}
\newcommand{\cF}{\mathcal{F}}
\newcommand{\cH}{\mathcal{H}}

\newcommand{\realset}{\mathbb{R}}

\newcommand{\colvar}{\textsc{v}}

\newcommand{\rowvar}{\textsc{u}}

\newcommand{\ceils}[1]{\left\lceil#1\right\rceil}
\newcommand{\condE}[2]{\mathbb{E} \left[#1 \,\middle|\, #2\right]}

\newcommand{\E}[1]{\mathbb{E} \left[#1\right]}

\newcommand{\I}[1]{\mathds{1} \! \left\{#1\right\}}

\newcommand{\rnd}[1]{\mathbf{#1}}
\newcommand{\set}[1]{\left\{#1\right\}}

\DeclareMathOperator*{\argmax}{arg\,max\,}

\mathchardef\mhyphen="2D

\newcommand{\bilinucb}{{\tt Rank1Elim}}

\newcommand{\glmucb}{{\tt GLM\mhyphen UCB}}

\newcommand{\linucb}{{\tt LinUCB}}

\newcommand{\ucb}{{\tt UCB1}}

\begin{document}

\runningtitle{Stochastic Rank-$1$ Bandits}
\runningauthor{Sumeet Katariya, Branislav Kveton, Csaba Szepesv\'ari, Claire Vernade, and Zheng Wen}

\twocolumn[
\aistatstitle{Stochastic Rank-$1$ Bandits}
\vspace{-0.1in}
\aistatsauthor{Sumeet Katariya \And Branislav Kveton \And Csaba Szepesv\'ari}
\aistatsaddress{Department of ECE \\ University of Wisconsin-Madison \\ \emph{katariya@wisc.edu} \And
Adobe Research \\ San Jose, CA \\ \emph{kveton@adobe.com} \And
Department of Computing Science \\ University of Alberta \\ \emph{szepesva@cs.ualberta.ca}}
\vspace{-0.05in}
\aistatsauthor{Claire Vernade \And Zheng Wen}
\aistatsaddress{Telecom ParisTech \\ Paris, France \\ \emph{claire.vernade@telecom-paristech.fr} \And
Adobe Research \\ San Jose, CA \\ \emph{zwen@adobe.com}}
\vspace{0.05in}]

\begin{abstract}
We propose stochastic rank-$1$ bandits, a class of online learning problems where at each step a learning agent chooses a pair of row and column arms, and receives the product of their values as a reward. The main challenge of the problem is that the individual values of the row and column are unobserved. We assume that these values are stochastic and drawn independently. We propose a computationally-efficient algorithm for solving our problem, which we call $\bilinucb$. We derive a $O((K + L) (1 / \Delta) \log n)$ upper bound on its $n$-step regret, where $K$ is the number of rows, $L$ is the number of columns, and $\Delta$ is the minimum of the row and column gaps; under the assumption that the mean row and column rewards are bounded away from zero. To the best of our knowledge, we present the first bandit algorithm that finds the maximum entry of a rank-$1$ matrix whose regret is linear in $K + L$, $1 / \Delta$, and $\log n$. We also derive a nearly matching lower bound. Finally, we evaluate $\bilinucb$ empirically on multiple problems. We observe that it leverages the structure of our problems and can learn near-optimal solutions even if our modeling assumptions are mildly violated.
\end{abstract}

%!TEX root = Paper.tex

\section{Introduction}
\label{sec:introduction}

We study the problem of finding the maximum entry of a stochastic rank-$1$ matrix from noisy and adaptively-chosen observations. This problem is motivated by two problems, ranking in the position-based model \cite{richardson07predicting} and online advertising.

The \emph{position-based model (PBM)} \cite{richardson07predicting} is one of the most fundamental click models \cite{chuklin15click}, a model of how people click on a list of $K$ items out of $L$. This model is defined as follows. Each \emph{item} is associated with its \emph{attraction} and each \emph{position} in the list is associated with its \emph{examination}. The attraction of any item and the examination of any position are i.i.d. Bernoulli random variables. The item in the list is \emph{clicked} only if it is attractive and its position is examined. Under these assumptions, the pair of the item and position that maximizes the probability of clicking is the maximum entry of a rank-$1$ matrix, which is the outer product of the attraction probabilities of items and the examination probabilities of positions.

As another example, consider a marketer of a product who has two sets of actions, $K$ population \emph{segments} and $L$ marketing \emph{channels}. Given a product, some segments are \emph{easier to market to} and some channels are \emph{more appropriate}. Now suppose that the conversion happens only if both actions are successful and that the successes of these actions are independent. Then similarly to our earlier example, the pair of the population segment and marketing channel that maximizes the conversion rate is the maximum entry of a rank-$1$ matrix.

We propose an online learning model for solving our motivating problems, which we call a \emph{stochastic rank-$1$ bandit}. The learning agent interacts with our problem as follows. At time $t$, the agent selects a pair of row and column arms, and receives the product of their individual values as a reward. The values are stochastic, drawn independently, and not observed. The goal of the agent is to maximize its expected cumulative reward, or equivalently to minimize its expected cumulative regret with respect to the optimal solution, the most rewarding pair of row and column arms.

We make five contributions. First, we precisely formulate the online learning problem of \emph{stochastic rank-$1$ bandits}. Second, we design an elimination algorithm for solving it, which we call $\bilinucb$. The key idea in $\bilinucb$ is to explore all remaining rows and columns randomly over all remaining columns and rows, respectively, to estimate their expected rewards; and then eliminate those rows and columns that seem suboptimal. This algorithm is computationally efficient and easy to implement. Third, we derive a $O((K + L) (1 / \Delta) \log n)$ gap-dependent upper bound on its $n$-step regret, where $K$ is the number of rows, $L$ is the number of columns, and $\Delta$ is the minimum of the row and column gaps; under the assumption that the mean row and column rewards are bounded away from zero. Fourth, we derive a nearly matching gap-dependent lower bound. Finally, we evaluate our algorithm empirically. In particular, we validate the scaling of its regret, compare it to multiple baselines, and show that it can learn near-optimal solutions even if our modeling assumptions are mildly violated.

We denote random variables by boldface letters and define $[n] = \set{1, \dots, n}$. For any sets $A$ and $B$, we denote by $A^B$ the set of all vectors whose entries are indexed by $B$ and take values from $A$.

%!TEX root = Paper.tex

\section{Setting}
\label{sec:setting}

We formulate our online learning problem as a \emph{stochastic rank-$1$ bandit}. An instance of this problem is defined by a tuple $(K, L, P_\rowvar, P_\colvar)$, where $K$ is the number of rows, $L$ is the number of columns, $P_\rowvar$ is a probability distribution over a unit hypercube $[0, 1]^K$, and $P_\colvar$ is a probability distribution over a unit hypercube $[0, 1]^L$.

Let $(\rnd{u}_t)_{t = 1}^n$ be an i.i.d. sequence of $n$ vectors drawn from distribution $P_\rowvar$ and $(\rnd{v}_t)_{t = 1}^n$ be an i.i.d. sequence of $n$ vectors drawn from distribution $P_\colvar$, such that $\rnd{u}_t$ and $\rnd{v}_t$ are drawn independently at any time $t$. The learning agent interacts with our problem as follows. At time $t$, it chooses \emph{arm} $(\rnd{i}_t, \rnd{j}_t) \in [K] \times [L]$ based on its history up to time $t$; and then \emph{observes} $\rnd{u}_t(\rnd{i}_t) \rnd{v}_t(\rnd{j}_t)$, which is also its \emph{reward}.

The goal of the agent is to maximize its expected cumulative reward in $n$ steps. This is equivalent to minimizing the \emph{expected cumulative regret} in $n$ steps
\begin{align*}
  R(n) = \E{\sum_{t = 1}^n R(\rnd{i}_t, \rnd{j}_t, \rnd{u}_t, \rnd{v}_t)}\,,
\end{align*}
where $R(\rnd{i}_t, \rnd{j}_t, \rnd{u}_t, \rnd{v}_t) = \rnd{u}_t(i^\ast) \rnd{v}_t(j^\ast) - \rnd{u}_t(\rnd{i}_t) \rnd{v}_t(\rnd{j}_t)$ is the \emph{instantaneous stochastic regret} of the agent at time $t$ and
\begin{align*}
  (i^\ast, j^\ast) = \argmax_{(i, j) \in [K] \times [L]} \E{\rnd{u}_1(i) \rnd{v}_1(j)}
\end{align*}
is the \emph{optimal solution} in hindsight of knowing $P_\rowvar$ and $P_\colvar$. Since $\rnd{u}_1$ and $\rnd{v}_1$ are drawn independently, and $\rnd{u}_1(i) \geq 0$ for all $i \in [K]$ and $\rnd{v}_1(j) \geq 0$ for all $j \in [L]$, we get that
\begin{align*}
  i^\ast = \argmax_{i \in [K]} \mu \bar{u}(i)\,, \quad
  j^\ast = \argmax_{j \in [L]} \mu \bar{v}(j)\,,
\end{align*}
for any $\mu > 0$, where $\bar{u} = \E{\rnd{u}_1}$ and $\bar{v} = \E{\rnd{v}_1}$. This is the key idea in our solution.

Note that the problem of learning $\bar{u}$ and $\bar{v}$ from stochastic observations $\set{\rnd{u}_t(\rnd{i}_t) \rnd{v}_t(\rnd{j}_t)}_{t = 1}^n$ is a special case of \emph{matrix completion from noisy observations} \cite{keshavan10matrix}. This problem is harder than that of learning $(i^\ast, j^\ast)$. In particular, the most popular approach to matrix completion is alternating minimization of a non-convex function \cite{koren09matrix}, where the observations are corrupted with Gaussian noise. In contrast, our proposed algorithm is guaranteed to learn the optimal solution with a high probability, and does not make any strong assumptions on $P_\rowvar$ and $P_\colvar$.

%!TEX root = Paper.tex

\section{Naive Solutions}
\label{sec:naive solutions}

Our learning problem is a $K L$-arm bandit with $K + L$ parameters, $\bar{u} \in [0, 1]^K$ and $\bar{v} \in [0, 1]^L$. The main challenge is to leverage this structure to learn efficiently. In this section, we discuss the challenges of solving our problem by existing algorithms. We conclude that a new algorithm is necessary and present it in \cref{sec:algorithm}.

Any rank-$1$ bandit is a multi-armed bandit with $K L$ arms. As such, it can be solved by $\ucb$ \cite{auer02finitetime}. The $n$-step regret of $\ucb$ in rank-$1$ bandits is $O(K L (1 / \Delta) \log n)$. Therefore, $\ucb$ is impractical when both $K$ and $L$ are large.

Note that $\log(\bar{u}(i) \bar{v}(j)) = \log(\bar{u}(i)) + \log(\bar{v}(j))$ for any $\bar{u}(i), \bar{v}(j) > 0$. Therefore, a rank-$1$ bandit can be viewed as a stochastic linear bandit and solved by $\linucb$ \cite{dani08stochastic,abbasi-yadkori11improved}, where the reward of arm $(i, j)$ is $\log(\rnd{u}_t(i)) + \log(\rnd{v}_t(j))$ and its features $x_{i, j} \in \set{0, 1}^{K + L}$ are
\begin{align}
  x_{i, j}(e) =
  \begin{cases}
    \I{e = i}, & e \leq K\,; \\
    \I{e - K = j}, & e > K\,,
  \end{cases}
  \label{eq:feature vector}
\end{align}
for any $e \in [K + L]$. This approach is problematic for at least two reasons. First, the reward is not properly defined when either $\rnd{u}_t(i) = 0$ or $\rnd{v}_t(j) = 0$. Second,
\begin{align*}
  \E{\log(\rnd{u}_t(i)) + \log(\rnd{v}_t(j))} \neq \log(\bar{u}(i)) + \log(\bar{v}(j))\,.
\end{align*}
Nevertheless, note that both sides of the above inequality have maxima at $(i^\ast, j^\ast)$, and therefore $\linucb$ should perform well. We compare to it in \cref{sec:comparison experiment}.

Also note that $\bar{u}(i) \bar{v}(j) = \exp[\log(\bar{u}(i)) + \log(\bar{v}(j))]$ for $\bar{u}(i), \bar{v}(j) > 0$. Therefore, a rank-$1$ bandit can be viewed as a generalized linear bandit and solved by $\glmucb$ \cite{filippi10parametric}, where the mean function is $\exp[\cdot]$ and the feature vector of arm $(i, j)$ is in \eqref{eq:feature vector}. This approach is not practical for three reasons. First, the parameter space is unbounded, because $\log(\bar{u}(i)) \to - \infty$ as $\bar{u}(i) \to 0$ and $\log(\bar{v}(j)) \to - \infty$ as $\bar{v}(j) \to 0$. Second, the confidence intervals of $\glmucb$ are scaled by the reciprocal of the minimum derivative of the mean function $c_\mu^{-1}$, which can be very large in our setting. In particular, $c_\mu = \min_{(i, j) \in [K] \times [L]} \bar{u}(i) \bar{v}(j)$. In addition, the gap-dependent upper bound on the regret of $\glmucb$ is $O((K + L)^2 c_\mu^{- 2})$, which further indicates that $\glmucb$ is not practical. Our upper bound in \cref{thm:upper bound} scales much better with all quantities of interest. Third, $\glmucb$ needs to compute the maximum-likelihood estimates of $\bar{u}$ and $\bar{v}$ at each step, which is a non-convex optimization problem (\cref{sec:setting}).

Some variants of our problem can be solved trivially. For instance, let $\rnd{u}_t(i) \in \set{0.1, 0.5}$ for all $i \in [K]$ and $\rnd{v}_t(j) \in \set{0.5, 0.9}$ for all $j \in [L]$. Then $(\rnd{u}_t(i), \rnd{v}_t(j))$ can be identified from $\rnd{u}_t(i) \rnd{v}_t(j)$, and the learning problem does not seem more difficult than a stochastic combinatorial semi-bandit \cite{kveton15tight}. We do not focus on such degenerate cases in this paper.

%!TEX root = Paper.tex

\section{$\bilinucb$ Algorithm}
\label{sec:algorithm}

\begin{algorithm}[t!]
  \caption{$\bilinucb$ for stochastic rank-$1$ bandits.}
  \label{alg:main}
  \begin{algorithmic}[1]
    \STATE // Initialization
    \STATE $t \gets 1$, \ $\tilde{\Delta}_0 \gets 1$, \
    $\rnd{C}^\rowvar_0 \gets \set{0}^{K \times L}$, \ $\rnd{C}^\colvar_0 \gets \set{0}^{K \times L}$,
    \STATE $\rnd{h}^\rowvar_0 \gets (1, \dots, K)$, \ $\rnd{h}^\colvar_0 \gets (1, \dots, L)$, \ $n_{-1} \gets 0$
    \STATE 
    \FORALL{$\ell = 0, 1, \dots$}
      \STATE $n_\ell \gets \ceils{4 \tilde{\Delta}_\ell^{-2} \log n}$
      \STATE $\rnd{I}_\ell \gets \bigcup_{i \in [K]} \set{\rnd{h}^\rowvar_\ell(i)}$, \
      $\rnd{J}_\ell \gets \bigcup_{j \in [L]} \set{\rnd{h}^\colvar_\ell(j)}$
      \STATE
      \STATE // Row and column exploration
      \FOR{$n_\ell - n_{\ell - 1}$ times}
        \STATE Choose uniformly at random column $j \in [L]$
        \STATE $j \gets \rnd{h}^\colvar_\ell(j)$
        \FORALL{$i \in \rnd{I}_\ell$}
          \STATE $\rnd{C}^\rowvar_\ell(i, j) \gets \rnd{C}^\rowvar_\ell(i, j) + \rnd{u}_t(i) \rnd{v}_t(j)$
          \STATE $t \gets t + 1$
        \ENDFOR
        \STATE Choose uniformly at random row $i \in [K]$
        \STATE $i \gets \rnd{h}^\rowvar_\ell(i)$
        \FORALL{$j \in \rnd{J}_\ell$}
          \STATE $\rnd{C}^\colvar_\ell(i, j) \gets \rnd{C}^\colvar_\ell(i, j) + \rnd{u}_t(i) \rnd{v}_t(j)$
          \STATE $t \gets t + 1$
        \ENDFOR
      \ENDFOR
      \STATE
      \STATE // UCBs and LCBs on the expected rewards of all remaining rows and columns
      \FORALL{$i \in \rnd{I}_\ell$}
        \STATE $\displaystyle
        \rnd{U}^\rowvar_\ell(i) \gets \frac{1}{n_\ell} \sum_{j = 1}^L \rnd{C}^\rowvar_\ell(i, j) +
        \sqrt{\frac{\log n}{n_\ell}}$
        \STATE $\displaystyle
        \rnd{L}^\rowvar_\ell(i) \gets \frac{1}{n_\ell} \sum_{j = 1}^L \rnd{C}^\rowvar_\ell(i, j) -
        \sqrt{\frac{\log n}{n_\ell}}$
      \ENDFOR
      \FORALL{$j \in \rnd{J}_\ell$}
        \STATE $\displaystyle
        \rnd{U}^\colvar_\ell(j) \gets \frac{1}{n_\ell} \sum_{i = 1}^K \rnd{C}^\colvar_\ell(i, j) +
        \sqrt{\frac{\log n}{n_\ell}}$
        \STATE $\displaystyle
        \rnd{L}^\colvar_\ell(j) \gets \frac{1}{n_\ell} \sum_{i = 1}^K \rnd{C}^\colvar_\ell(i, j) -
        \sqrt{\frac{\log n}{n_\ell}}$
      \ENDFOR
      \STATE
      \STATE // Row and column elimination
      \STATE $\rnd{i}_\ell \gets \argmax_{i \in \rnd{I}_\ell} \rnd{L}^\rowvar_\ell(i)$
      \STATE $\rnd{h}^\rowvar_{\ell + 1} \gets \rnd{h}^\rowvar_\ell$
      \FORALL{$i = 1, \dots, K$}
        \IF{$\rnd{U}^\rowvar_\ell(\rnd{h}^\rowvar_\ell(i)) \leq \rnd{L}^\rowvar_\ell(\rnd{i}_\ell)$}
          \STATE $\rnd{h}^\rowvar_{\ell + 1}(i) \gets \rnd{i}_\ell$
        \ENDIF
      \ENDFOR
      \STATE
      \STATE $\rnd{j}_\ell \gets \argmax_{j \in \rnd{J}_\ell} \rnd{L}^\colvar_\ell(j)$
      \STATE $\rnd{h}^\colvar_{\ell + 1} \gets \rnd{h}^\colvar_\ell$
      \FORALL{$j = 1, \dots, L$}
        \IF{$\rnd{U}^\colvar_\ell(\rnd{h}^\colvar_\ell(j)) \leq \rnd{L}^\colvar_\ell(\rnd{j}_\ell)$}
          \STATE $\rnd{h}^\colvar_{\ell + 1}(j) \gets \rnd{j}_\ell$
        \ENDIF
      \ENDFOR
      \STATE
      \STATE $\tilde{\Delta}_{\ell + 1} \gets \tilde{\Delta}_\ell / 2$, \
      $\rnd{C}^\rowvar_{\ell + 1} \gets \rnd{C}^\rowvar_\ell$, \
      $\rnd{C}^\colvar_{\ell + 1} \gets \rnd{C}^\colvar_\ell$
    \ENDFOR
  \end{algorithmic}
\end{algorithm}

Our algorithm, $\bilinucb$, is shown in \cref{alg:main}. It is an elimination algorithm \cite{auer10ucb}, which maintains $\ucb$ confidence intervals \cite{auer02finitetime} on the expected rewards of all rows and columns. $\bilinucb$ operates in stages, which quadruple in length. In each stage, it explores all remaining rows and columns randomly over all remaining columns and rows, respectively. At the end of the stage, it eliminates all rows and columns that cannot be optimal.

The eliminated rows and columns are tracked as follows. We denote by $\rnd{h}^\rowvar_\ell(i)$ the index of the most rewarding row whose expected reward is believed by $\bilinucb$ to be at least as high as that of row $i$ in stage $\ell$. Initially, $\rnd{h}^\rowvar_0(i) = i$. When row $i$ is eliminated by row $\rnd{i}_\ell$ in stage $\ell$, $\rnd{h}^\rowvar_{\ell + 1}(i)$ is set to $\rnd{i}_\ell$; then when row $\rnd{i}_\ell$ is eliminated by row $\rnd{i}_{\ell'}$ in stage $\ell' > \ell$, $\rnd{h}^\rowvar_{\ell' + 1}(i)$ is set to $\rnd{i}_{\ell'}$; and so on. The corresponding column quantity, $\rnd{h}^\colvar_\ell(j)$, is defined and updated analogously. The \emph{remaining rows and columns in stage $\ell$}, $\rnd{I}_\ell$ and $\rnd{J}_\ell$, are then the unique values in $\rnd{h}^\rowvar_\ell$ and $\rnd{h}^\colvar_\ell$, respectively; and we set these in line $7$ of \cref{alg:main}.

Each stage of \cref{alg:main} has two main steps: exploration (lines $9$--$20$) and elimination (lines $22$--$41$). In the row exploration step, each row $i \in \rnd{I}_\ell$ is explored randomly over all remaining columns $\rnd{J}_\ell$ such that its expected reward up to stage $\ell$ is at least $\mu \bar{u}(i)$, where $\mu$ is in \eqref{eq:average reward}. To guarantee this, we sample column $j \in [L]$ randomly and then substitute it with column $\rnd{h}^\colvar_\ell(j)$, which is at least as rewarding as column $j$. This is critical to avoid $1 / \min_{j \in [L]} \bar{v}(j)$ in our regret bound, which can be large and is not necessary. The observations are stored in \emph{reward matrix} $\rnd{C}^\rowvar_\ell \in \realset^{K \times L}$. As all rows are explored similarly, their expected rewards are scaled similarly, and this permits elimination. The column exploration step is analogous.

In the elimination step, the confidence intervals of all remaining rows, $[\rnd{L}^\rowvar_\ell(i), \rnd{U}^\rowvar_\ell(i)]$ for any $i \in \rnd{I}_\ell$, are estimated from matrix $\rnd{C}^\rowvar_\ell \in \realset^{K \times L}$; and the confidence intervals of all remaining columns, $[\rnd{L}^\colvar_\ell(j), \rnd{U}^\colvar_\ell(j)]$ for any $j \in \rnd{J}_\ell$, are estimated from $\rnd{C}^\colvar_\ell \in \realset^{K \times L}$. This separation is needed to guarantee that the expected rewards of all remaining rows and columns are scaled similarly. The confidence intervals are designed such that
\begin{align*}
  \rnd{U}^\rowvar_\ell(i) \leq
  \rnd{L}^\rowvar_\ell(\rnd{i}_\ell) =
  \max_{i \in \rnd{I}_\ell} \rnd{L}^\rowvar_\ell(i)
\end{align*}
implies that row $i$ is suboptimal with a high probability for any column elimination policy up to the end of stage $\ell$, and
\begin{align*}
  \rnd{U}^\colvar_\ell(j) \leq
  \rnd{L}^\colvar_\ell(\rnd{j}_\ell) =
  \max_{j \in \rnd{J}_\ell} \rnd{L}^\colvar_\ell(j)
\end{align*}
implies that column $j$ is suboptimal with a high probability for any row elimination policy up to the end of stage $\ell$. As a result, all suboptimal rows and columns are eliminated correctly with a high probability.

%!TEX root = Paper.tex

\section{Analysis}
\label{sec:analysis}

This section has three subsections. In \cref{sec:upper bound}, we derive a gap-dependent upper bound on the $n$-step regret of $\bilinucb$. In \cref{sec:lower bound}, we derive a gap-dependent lower bound that nearly matches our upper bound. In \cref{sec:discussion}, we discuss the results of our analysis.

\subsection{Upper Bound}
\label{sec:upper bound}

The hardness of our learning problem is measured by two sets of metrics. The first metrics are gaps. The \emph{gaps} of row $i \in [K]$ and column $j \in [L]$ are defined as
\begin{align}
  \Delta^\rowvar_i = \bar{u}(i^\ast) - \bar{u}(i)\,, \quad
  \Delta^\colvar_j = \bar{v}(j^\ast) - \bar{v}(j)\,,
  \label{eq:gaps}
\end{align}
respectively; and the \emph{minimum row and column gaps} are defined as
\begin{align}
  \Delta^\rowvar_{\min} = \!\!\! \min_{i \in [K]: \Delta^\rowvar_i > 0} \Delta^\rowvar_i\,, \quad
  \Delta^\colvar_{\min} = \!\!\! \min_{j \in [L]: \Delta^\colvar_j > 0} \Delta^\colvar_j\,,
  \label{eq:minimum gaps}
\end{align}
respectively. Roughly speaking, the smaller the gaps, the harder the problem. The second metric is the minimum of the average of entries in $\bar{u}$ and $\bar{v}$, which is defined as
\begin{align}
  \mu = \min \set{\frac{1}{K} \sum_{i = 1}^K \bar{u}(i), \ \frac{1}{L} \sum_{j = 1}^L \bar{v}(j)}\,.
  \label{eq:average reward}
\end{align}
The smaller the value of $\mu$, the harder the problem. This quantity appears in our regret bound due to the averaging character of $\bilinucb$ (\cref{sec:algorithm}). Our upper bound on the regret of $\bilinucb$ is stated and proved below.

\begin{theorem}
\label{thm:upper bound} The expected $n$-step regret of $\bilinucb$ is bounded as
\begin{align*}
  R(n) \leq
  \frac{1}{\mu^2} \left(\sum_{i = 1}^K \frac{384}{\bar{\Delta}^\rowvar_i} +
  \sum_{j = 1}^L \frac{384}{\bar{\Delta}^\colvar_j}\right) \log n +
  3 (K + L)\,,
\end{align*}
where
\begin{align*}
  \bar{\Delta}^\rowvar_i & = \Delta^\rowvar_i + \I{\Delta^\rowvar_i = 0} \Delta^\colvar_{\min}\,, \\
  \bar{\Delta}^\colvar_j & = \Delta^\colvar_j + \I{\Delta^\colvar_j = 0} \Delta^\rowvar_{\min}\,.
\end{align*}
\end{theorem}

The proof of \cref{thm:upper bound} is organized as follows. First, we bound the probability that at least one confidence interval is violated. The corresponding regret is small, $O(K + L)$. Second, by the design of $\bilinucb$ and because all confidence intervals hold, the expected reward of any row $i \in \allowbreak [K]$ is at least $\mu \bar{u}(i)$. Because all rows are explored in the same way, any suboptimal row $i$ is guaranteed to be eliminated after $O([1 / (\mu \Delta^\rowvar_i)^2] \log n)$ observations. Third, we factorize the regret due to exploring row $i$ into its row and column components, and bound both of them. This is possible because $\bilinucb$ eliminates rows and columns simultaneously. Finally, we sum up the regret of all explored rows and columns.

Note that the gaps in \cref{thm:upper bound}, $\bar{\Delta}^\rowvar_i$ and $\bar{\Delta}^\colvar_j$, are slightly different from those in \eqref{eq:gaps}. In particular, all zero row and column gaps in \eqref{eq:gaps} are substituted with the minimum column and row gaps, respectively. The reason is that the regret due to exploring optimal rows and columns is positive until all suboptimal columns and rows are eliminated, respectively. The proof of \cref{thm:upper bound} is below.

\begin{proof}
Let $\rnd{R}^\rowvar_\ell(i)$ and $\rnd{R}^\colvar_\ell(j)$ be the stochastic regret associated with exploring row $i$ and column $j$, respectively, in stage $\ell$. Then the expected $n$-step regret of $\bilinucb$ is bounded as
\begin{align*}
  R(n) \leq
  \E{\sum_{\ell = 0}^{n - 1} \left(\sum_{i = 1}^K \rnd{R}^\rowvar_\ell(i) +
  \sum_{j = 1}^L \rnd{R}^\colvar_\ell(j)\right)}\,,
\end{align*}
where the outer sum is over possibly $n$ stages. Let
\begin{align*}
  \bar{\rnd{u}}_\ell(i)
  & = \sum_{t = 0}^\ell \condE{\sum_{j = 1}^L
  \frac{\rnd{C}^\rowvar_t(i, j) - \rnd{C}^\rowvar_{t - 1}(i, j)}{n_\ell}}{\rnd{h}^\colvar_t} \\
  & = \bar{u}(i) \sum_{t = 0}^\ell \frac{n_t - n_{t - 1}}{n_\ell} \sum_{j = 1}^L \frac{\bar{v}(\rnd{h}^\colvar_t(j))}{L}
\end{align*}
be the expected reward of row $i \in \rnd{I}_\ell$ in the first $\ell$ stages, where $n_{-1} = 0$ and $\rnd{C}^\rowvar_{-1}(i, j) = 0$; and let
\begin{align*}
  \cE^\rowvar_\ell =
  \set{\forall i \in \rnd{I}_\ell: \bar{\rnd{u}}_\ell(i) \in [\rnd{L}^\rowvar_\ell(i), \rnd{U}^\rowvar_\ell(i)], \
  \bar{\rnd{u}}_\ell(i) \geq \mu \bar{u}(i)}
\end{align*}
be the event that for all remaining rows $i \in \rnd{I}_\ell$ at the end of stage $\ell$, the confidence interval on the expected reward holds and that this reward is at least $\mu \bar{u}(i)$. Let $\overline{\cE^\rowvar_\ell}$ be the complement of event $\cE^\rowvar_\ell$. Let
\begin{align*}
  \bar{\rnd{v}}_\ell(j)
  & = \sum_{t = 0}^\ell \condE{\sum_{i = 1}^K
  \frac{\rnd{C}^\colvar_t(i, j) - \rnd{C}^\colvar_{t - 1}(i, j)}{n_\ell}}{\rnd{h}^\rowvar_t} \\
  & = \bar{v}(j) \sum_{t = 0}^\ell \frac{n_t - n_{t - 1}}{n_\ell} \sum_{i = 1}^K \frac{\bar{u}(\rnd{h}^\rowvar_t(i))}{K}
\end{align*}
denote the expected reward of column $j \in \rnd{J}_\ell$ in the first $\ell$ stages, where $n_{-1} = 0$ and $\rnd{C}^\colvar_{-1}(i, j) = 0$; and let
\begin{align*}
  \cE^\colvar_\ell =
  \set{\forall j \in \rnd{J}_\ell: \bar{\rnd{v}}_\ell(j) \in [\rnd{L}^\colvar_\ell(j), \rnd{U}^\colvar_\ell(j)], \
  \bar{\rnd{v}}_\ell(j) \geq \mu \bar{v}(j)}
\end{align*}
be the event that for all remaining columns $j \in \rnd{J}_\ell$ at the end of stage $\ell$, the confidence interval on the expected reward holds and that this reward is at least $\mu \bar{v}(j)$. Let $\overline{\cE^\colvar_\ell}$ be the complement of event $\cE^\colvar_\ell$. Let $\cE$ be the event that all events $\cE^\rowvar_\ell$ and $\cE^\colvar_\ell$ happen; and $\ccE$ be the complement of $\cE$, the event that at least one of $\cE^\rowvar_\ell$ and $\cE^\colvar_\ell$ does not happen. Then the expected $n$-step regret of $\bilinucb$ is bounded from above as
\begin{align*}
  R(n)
  \leq {} & \E{\left(\sum_{\ell = 0}^{n - 1} \left(\sum_{i = 1}^K \rnd{R}^\rowvar_\ell(i) +
  \sum_{j = 1}^L \rnd{R}^\colvar_\ell(j)\right)\right) \I{\cE}} + {} \\
  & n P(\ccE) \\
  \leq {} & \sum_{i = 1}^K \E{\sum_{\ell = 0}^{n - 1} \rnd{R}^\rowvar_\ell(i) \I{\cE}} + {} \\
  & \sum_{j = 1}^L \E{\sum_{\ell = 0}^{n - 1} \rnd{R}^\colvar_\ell(j) \I{\cE}} + 2 (K + L)\,,
\end{align*}
where the last inequality is from \cref{lem:bad events} in \cref{sec:lemmas}.

Let $\cH_\ell = (\rnd{I}_\ell, \rnd{J}_\ell)$ be the rows and columns in stage $\ell$, and
\begin{align*}
  \mathcal{F}_\ell =
  \set{\forall i \in \rnd{I}_\ell, j \in \rnd{J}_\ell: \Delta^\rowvar_i \leq \frac{2 \tilde{\Delta}_{\ell - 1}}{\mu}, \
  \Delta^\colvar_j \leq \frac{2 \tilde{\Delta}_{\ell - 1}}{\mu}}
\end{align*}
be the event that all rows and columns with ``large gaps'' are eliminated by the beginning of stage $\ell$. By \cref{lem:elimination} in \cref{sec:lemmas}, event $\cE$ causes event $\mathcal{F}_\ell$. Now note that the expected regret in stage $\ell$ is independent of $\mathcal{F}_\ell$ given $\cH_\ell$. Therefore, the regret can be further bounded as
\begin{align}
  R(n)
  \leq {} & \sum_{i = 1}^K \E{\sum_{\ell = 0}^{n - 1} \condE{\rnd{R}^\rowvar_\ell(i)}{\cH_\ell} \I{\cF_\ell}} + {}
  \label{eq:component regret} \\
  & \sum_{j = 1}^L \E{\sum_{\ell = 0}^{n - 1} \condE{\rnd{R}^\colvar_\ell(j)}{\cH_\ell} \I{\cF_\ell}} + {}
  \nonumber \\
  & 2 (K + L)\,.
  \nonumber
\end{align}
By \cref{lem:row regret} in \cref{sec:lemmas},
\begin{align*}
  \E{\sum_{\ell = 0}^{n - 1} \condE{\rnd{R}^\rowvar_\ell(i)}{\cH_\ell} \I{\cF_\ell}}
  & \leq \frac{384}{\mu^2 \bar{\Delta}^\rowvar_i} \log n + 1\,, \\
  \E{\sum_{\ell = 0}^{n - 1} \condE{\rnd{R}^\colvar_\ell(j)}{\cH_\ell} \I{\cF_\ell}}
  & \leq \frac{384}{\mu^2 \bar{\Delta}^\colvar_j} \log n + 1\,,
\end{align*}
for any row $i \in [K]$ and column $j \in [L]$. Finally, we apply the above upper bounds to \eqref{eq:component regret} and get our main claim.
\end{proof}

%!TEX root = Paper.tex

\subsection{Lower Bound}
\label{sec:lower bound}

We derive a gap-dependent lower bound on the family of rank-$1$ bandits where $P_\rowvar$ and $P_\colvar$ are products of independent Bernoulli variables, which are parameterized by their means $\bar{u}$ and $\bar{v}$, respectively. The lower bound is derived for any \emph{uniformly efficient algorithm} $\mathcal{A}$, which is any algorithm such that for any $(\bar{u}, \bar{v}) \in [0, 1]^K \times [0, 1]^L$ and any $\alpha \in (0, 1)$, $R(n) = o(n^\alpha)$.

\begin{theorem}
\label{thm:lower bound} For any problem $(\bar{u}, \bar{v}) \in [0, 1]^K \times [0, 1]^L$ with a unique best arm and any uniformly efficient algorithm $\mathcal{A}$ whose regret is $R(n)$,
\begin{align*}
  \liminf_{n \to \infty} \frac{R(n)}{\log n}
  \geq {} & \sum_{i \in [K] \setminus \set{i^\ast}} \frac{\bar{u}(i^\ast) \bar{v}(j^\ast) - \bar{u}(i) \bar{v}(j^\ast)}
  {d(\bar{u}(i) \bar{v}(j^\ast), \bar{u}(i^\ast) \bar{v}(j^\ast))} + {} \\
  & \sum_{j \in [L] \setminus \set{j^\ast}} \frac{\bar{u}(i^\ast) \bar{v}(j^\ast) - \bar{u}(i^\ast) \bar{v}(j)}
  {d(\bar{u}(i^\ast) \bar{v}(j), \bar{u}(i^\ast) \bar{v}(j^\ast))}\,,
\end{align*}
where $d(p, q)$ is the \emph{Kullback-Leibler (KL) divergence} between Bernoulli random variables with means $p$ and $q$.
\end{theorem}

The lower bound involves two terms. The first term is the regret due to learning the optimal row $i^\ast$, while playing the optimal column $j^\ast$. The second term is the regret due to learning the optimal column $j^\ast$, while playing the optimal row $i^\ast$. We do not know whether this lower bound is tight. We discuss its tightness in \cref{sec:discussion}.

\begin{proof}
The proof is based on the change-of-measure techniques from Kaufmann \etal~\cite{kaufmann16complexity} and Lagree \etal~\cite{lagree16multipleplay}, who ultimately build on Graves and Lai \cite{graves97asymptotically}. Let
\begin{align*}
  \textstyle
  w^\ast(\bar{u}, \bar{v}) = \max_{(i, j) \in [K] \times [L]} \bar{u}(i) \bar{v}(j)
\end{align*}
be the maximum reward in model $(\bar{u}, \bar{v})$. We consider the set of models where $\bar{u}(i^\ast)$ and $\bar{v}(j^\ast)$ remain the same, but the optimal arm changes,
\begin{align*}
  B(\bar{u}, \bar{v}) =
  \{&(\bar{u}', \bar{v}') \in [0, 1]^K \times [0, 1]^L: \bar{u}(i^\ast) = \bar{u}'(i^\ast), \\
  &\bar{v}(j^\ast) = \bar{v}'(j^\ast), \ w^\ast(\bar{u}, \bar{v}) < w^\ast(\bar{u}', \bar{v}')\}\,.
\end{align*}
By Theorem 17 of Kaufmann \etal~\cite{kaufmann16complexity},
\begin{align*}
  \liminf_{n \to \infty} \frac{\displaystyle \sum_{i = 1}^K \sum_{j = 1}^L
  \E{\rnd{T}_n(i, j)} d(\bar{u}(i) \bar{v}(j), \bar{u}'(i) \bar{v}'(j))}{\log n} \geq 1
\end{align*}
for any $(\bar{u}', \bar{v}') \in B(\bar{u}, \bar{v})$, where $\E{\rnd{T}_n(i, j)}$ is the expected number of times that arm $(i, j)$ is chosen in $n$ steps in problem $(\bar{u}, \bar{v})$. From this and the regret decomposition
\begin{align*}
  \textstyle
  R(n) = \sum_{i = 1}^K \sum_{j = 1}^L \E{\rnd{T}_n(i, j)} (\bar{u}(i^*) \bar{v}(j^*) - \bar{u}(i) \bar{v}(j))\,,
\end{align*}
we get that
\begin{align*}
  \liminf_{n \to \infty} \frac{R(n)}{\log n} \geq f(\bar{u}, \bar{v})\,,
\end{align*}
where
\begin{align*}
  f(\bar{u}, \bar{v})
  = \inf_{c \in \Theta} & \ \ \sum_{i = 1}^K \sum_{j = 1}^L
  (\bar{u}(i^\ast) \bar{v}(j^\ast) - \bar{u}(i) \bar{v}(j)) c_{i, j} \\
  \text{s.t.} & \ \ \ \forall (\bar{u}', \bar{v}') \in B(\bar{u}, \bar{v}): \\
  & \ \ \sum_{i = 1}^K \sum_{j = 1}^L d(\bar{u}(i) \bar{v}(j), \bar{u}'(i) \bar{v}'(j)) c_{i, j} \geq 1
\end{align*}
and $\Theta = [0, \infty)^{K \times L}$. To obtain our lower bound, we carefully relax the constraints of the above problem, so that we do not loose much in the bound. The details are presented in \cref{sec:LBdetails}. In the relaxed problem, only $K + L - 1$ entries in the optimal solution $c^\ast$ are non-zero, as in Combes \etal~\cite{combes15learning}, and they are
\begin{align*}
  c^\ast_{i, j} =
  \begin{cases}
    1 / d(\bar{u}(i) \bar{v}(j^\ast), \bar{u}(i^\ast) \bar{v}(j^\ast))\,, & j = j^\ast, i \ne i^\ast\,; \\
    1 / d(\bar{u}(i^\ast) \bar{v}(j), \bar{u}(i^\ast) \bar{v}(j^\ast))\,, & i = i^\ast, j \ne j^\ast \,; \\
    0\,, & \text{otherwise.}
  \end{cases}
\end{align*}
Now we substitute $c^\ast$ into the objective of the above problem and get our lower bound.
\end{proof}

%!TEX root = Paper.tex

\subsection{Discussion}
\label{sec:discussion}

We derive a gap-dependent upper bound on the $n$-step regret of $\bilinucb$ in \cref{thm:upper bound}, which is
\begin{align*}
  O((K + L) (1 / \mu^2) (1 / \Delta) \log n)\,,
\end{align*}
where $K$ denotes the number of rows, $L$ denotes the number of columns, $\Delta = \min \set{\Delta^\rowvar_{\min}, \Delta^\colvar_{\min}}$ is the minimum of the row and column gaps in \eqref{eq:minimum gaps}, and $\mu$ is the minimum of the average of entries in $\bar{u}$ and $\bar{v}$, as defined in \eqref{eq:average reward}.

We argue that our upper bound is nearly tight on the following class of problems. The $i$-th entry of $\rnd{u}_t$, $\rnd{u}_t(i)$, is an independent Bernoulli variable with mean
\begin{align*}
  \bar{u}(i) = p_\rowvar + \Delta_\rowvar \I{i = 1}
\end{align*}
for some $p_\rowvar \in [0, 1]$ and row gap $\Delta_\rowvar \in (0, 1 - p_\rowvar]$. The $j$-th entry of $\rnd{v}_t$, $\rnd{v}_t(j)$, is an independent Bernoulli variable with mean
\begin{align*}
  \bar{v}(j) = p_\colvar + \Delta_\colvar \I{j = 1}
\end{align*}
for $p_\colvar \in [0, 1]$ and column gap $\Delta_\colvar \in (0, 1 - p_\colvar]$. Note that the optimal arm is $(1, 1)$ and that the expected reward for choosing it is $(p_\rowvar + \Delta_\rowvar) (p_\colvar + \Delta_\colvar)$. We refer to the instance of this problem by $B_\textsc{spike}(K, L, p_\rowvar, p_\colvar, \Delta_\rowvar, \Delta_\colvar)$; and parameterize it by $K$, $L$, $p_\rowvar$, $p_\colvar$, $\Delta_\rowvar$, and $\Delta_\colvar$.

Let $p_\rowvar = 0.5 - \Delta_\rowvar$ for $\Delta_\rowvar \in [0, 0.25]$, and $p_\colvar = 0.5 - \Delta_\colvar$ for $\Delta_\colvar \in [0, 0.25]$. Then the upper bound in \cref{thm:upper bound} is
\begin{align*}
  O([K (1 / \Delta_\rowvar) + L (1 / \Delta_\colvar)] \log n)
\end{align*}
since $1 / \mu^2 \leq 1 / 0.25^2 = 16$. On the other hand, the lower bound in \cref{thm:lower bound} is
\begin{align*}
  \Omega([K (1 / \Delta_\rowvar) + L (1 / \Delta_\colvar)] \log n)
\end{align*}
since $d(p, q) \leq [q (1 - q)]^{-1} (p - q)^2$ and $q = 1 - q = 0.5$. Note that the bounds match in  $K$, $L$, the gaps, and $\log n$.

We conclude with the observation that $\bilinucb$ is suboptimal in problems where $\mu$ in \eqref{eq:average reward} is small. In particular, consider the above problem, and choose $\Delta_\rowvar = \Delta_\colvar = 0.5$ and $K = L$. In this problem, the regret of $\bilinucb$ is $O(K^3 \log n)$; because $\bilinucb$ eliminates $O(K)$ rows and columns with $O(1 / K)$ gaps, and the regret for choosing any suboptimal arm is $O(1)$. This is much higher than the regret of a naive solution by $\ucb$ in \cref{sec:naive solutions}, which would be $O(K^2 \log n)$. Note that the upper bound in \cref{thm:upper bound} is also $O(K^3 \log n)$. Therefore, it is not loose, and a new algorithm is necessary to improve over $\ucb$ in this particular problem.

%!TEX root = Paper.tex

\section{Experiments}
\label{sec:experiments}

\begin{table*}[t]
  \centering
  {\small
  \begin{tabular}{rrr} \hline
    $K$ & $L$ & Regret \\ \hline
    8 & 8 & $17491 \pm \phantom{0}384$ \\
    8 & 16 & $29628 \pm 1499$ \\
    8 & 32 & $50030 \pm 1931$ \\
    16 & 8 & $28862 \pm \phantom{0}585$ \\
    16 & 16 & $41823 \pm 1689$ \\
    16 & 32 & $62451 \pm 2268$ \\
    32 & 8 & $46156 \pm \phantom{0}806$ \\
    32 & 16 & $61992 \pm 2339$ \\
    32 & 32 & $85208 \pm 3546$ \\ \hline
    \vspace{-0.05in} \\
    \multicolumn{3}{c}{$p_\rowvar = p_\colvar = 0.7$, \ $\Delta_\rowvar = \Delta_\colvar = 0.2$}
  \end{tabular}
  }
  \hspace{0.2in}
  {\small
  \begin{tabular}{rrr} \hline
    $p_\rowvar$ & $p_\colvar$ & Regret \\ \hline
    0.700 & 0.700 & $17744 \pm \phantom{0}466$ \\
    0.700 & 0.350 & $23983 \pm \phantom{0}594$ \\
    0.700 & 0.175 & $24776 \pm 2333$ \\
    0.350 & 0.700 & $22963 \pm \phantom{0}205$ \\
    0.350 & 0.350 & $38373 \pm \phantom{00}71$ \\
    0.350 & 0.175 & $57401 \pm \phantom{00}68$ \\
    0.175 & 0.700 & $27440 \pm 2011$ \\
    0.175 & 0.350 & $57492 \pm \phantom{00}67$ \\
    0.175 & 0.175 & $95586 \pm \phantom{00}99$ \\ \hline
    \vspace{-0.05in} \\
    \multicolumn{3}{c}{$K = L = 8$, \ $\Delta_\rowvar = \Delta_\colvar = 0.2$}
  \end{tabular}
  }
  \hspace{0.2in}
  {\small
  \begin{tabular}{rrr} \hline
    $\Delta_\rowvar$ & $\Delta_\colvar$ & Regret \\ \hline
    0.20 & 0.20 & $17653 \pm \phantom{0}307$ \\
    0.20 & 0.10 & $22891 \pm \phantom{0}912$ \\
    0.20 & 0.05 & $30954 \pm \phantom{0}787$ \\
    0.10 & 0.20 & $20958 \pm \phantom{0}614$ \\
    0.10 & 0.10 & $33642 \pm 1089$ \\
    0.10 & 0.05 & $45511 \pm 3257$ \\
    0.05 & 0.20 & $30688 \pm \phantom{0}482$ \\
    0.05 & 0.10 & $44390 \pm 2542$ \\
    0.05 & 0.05 & $68412 \pm 2312$ \\ \hline
    \vspace{-0.05in} \\
    \multicolumn{3}{c}{$K = L = 8$, \ $p_\rowvar = p_\colvar = 0.7$}
  \end{tabular}
  }
  \caption{The $n$-step regret of $\bilinucb$ in $n = 2\text{M}$ steps as $K$ and $L$ increase (left), $p_\rowvar$ and $p_\colvar$ decrease (middle), and $\Delta_\rowvar$ and $\Delta_\colvar$ decrease (right). The results are averaged over $20$ runs.}
  \label{tab:scaling}
\end{table*}

\begin{figure*}[t]
  \centering
  \includegraphics[width=2.2in]{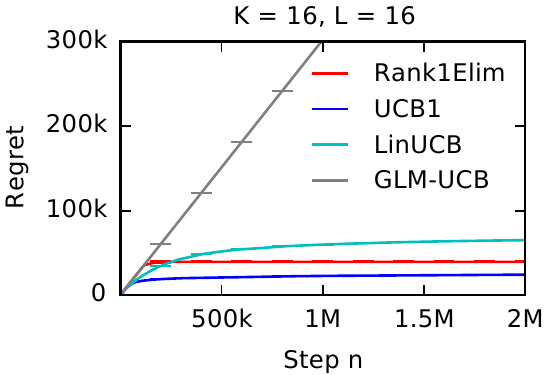}
  \includegraphics[width=2.2in]{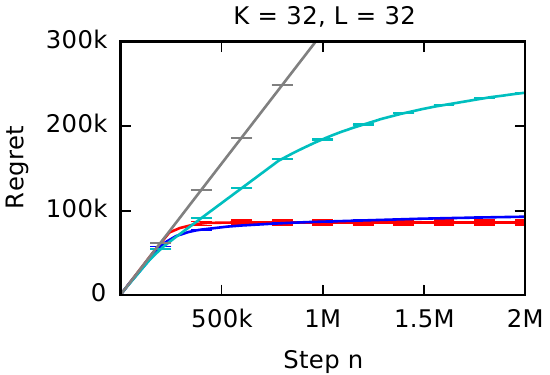}
  \includegraphics[width=2.2in]{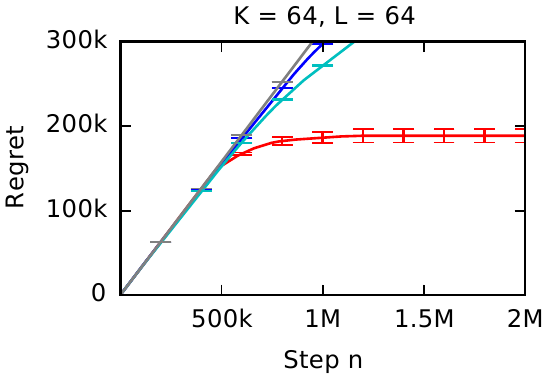}
  \caption{The $n$-step regret of $\bilinucb$, $\ucb$, $\linucb$, and $\glmucb$ on three synthetic problems in up to $n = 2\text{M}$ steps. The results are averaged over $20$ runs.}
  \label{fig:comparison}
\end{figure*}

We conduct three experiments. In \cref{sec:regret bound experiment}, we validate that the regret of $\bilinucb$ grows as suggested by \cref{thm:upper bound}. In \cref{sec:comparison experiment}, we compare $\bilinucb$ to three baselines. Finally, in \cref{sec:movielens experiment}, we evaluate $\bilinucb$ on a real-world problem where our modeling assumptions are violated.

\subsection{Regret Bound}
\label{sec:regret bound experiment}

The first experiment shows that the regret of $\bilinucb$ scales as suggested by our upper bound in \cref{thm:upper bound}. We experiment with the class of synthetic problems from \cref{sec:discussion}, $B_\textsc{spike}(K, L, p_\rowvar, p_\colvar, \Delta_\rowvar, \Delta_\colvar)$. We vary its parameters and report the $n$-step regret in $2$ million (M) steps.

\cref{tab:scaling} shows the $n$-step regret of $\bilinucb$ for various choices of $K$, $L$, $p_\rowvar$, $p_\colvar$, $\Delta_\rowvar$, and $\Delta_\colvar$. In each table, we vary two parameters and keep the rest fixed. We observe that the regret increases as $K$ and $L$ increase, and $\Delta_\rowvar$ and $\Delta_\colvar$ decrease; as suggested by \cref{thm:upper bound}. Specifically, the regret doubles when $K$ and $L$ are doubled, and when $\Delta_\rowvar$ and $\Delta_\colvar$ are halved. We also observe that the regret is not quadratic in $1 / \mu$, where $\mu \approx \min \set{p_\rowvar, p_\colvar}$. This indicates that the upper bound in \cref{thm:upper bound} is loose in $\mu$ when $\mu$ is bounded away from zero. We argue in \cref{sec:discussion} that this is not the case as $\mu \to 0$.

\subsection{Comparison to Alternative Solutions}
\label{sec:comparison experiment}

In the second experiment, we compare $\bilinucb$ to the three alternative methods in \cref{sec:naive solutions}: $\ucb$, $\linucb$, and $\glmucb$. The confidence radii of $\linucb$ and $\glmucb$ are set as suggested by Abbasi-Yadkori \etal~\cite{abbasi-yadkori11improved} and Filippi \etal~\cite{filippi10parametric}, respectively. The maximum-likelihood estimates of $\bar{u}$ and $\bar{v}$ in $\glmucb$ are computed using the online EM \cite{cappe09online}, which is observed to converge to $\bar{u}$ and $\bar{v}$ in our problems. We experiment with the problem from \cref{sec:regret bound experiment}, where $p_\rowvar = p_\colvar = 0.7$, $\Delta_\rowvar = \Delta_\colvar = 0.2$, and $K = L$.

Our results are reported in \cref{fig:comparison}. We observe that the regret of $\bilinucb$ flattens in all three problems, which indicates that $\bilinucb$ learns the optimal arm. When $K = 16$, $\ucb$ has a lower regret than $\bilinucb$. However, because the regret of $\ucb$ is $O(K L)$ and the regret of $\bilinucb$ is $O(K + L)$,  $\bilinucb$ can outperform $\ucb$ on larger problems. When $K = 32$, both algorithms already perform similarly; and when $K = 64$, $\bilinucb$ clearly outperforms $\ucb$. This shows that $\bilinucb$ can leverage the structure of our problem. Neither $\linucb$ nor $\glmucb$ are competitive on any of our problems.

We investigated the poor performance of both $\linucb$ and $\glmucb$. When the confidence radii of $\linucb$ are multiplied by $1 / 3$, $\linucb$ becomes competitive on all problems. When the confidence radii of $\glmucb$ are multiplied by $1 / 100$, $\glmucb$ is still not competitive on any of our problems. We conclude that $\linucb$ and $\glmucb$ perform poorly because their theory-suggested confidence intervals are too wide. In contrast, $\bilinucb$ is implemented with its theory-suggested intervals in all experiments.

\subsection{MovieLens Experiment}
\label{sec:movielens experiment}

\begin{figure*}[t]
  \centering
  \includegraphics[width=2.2in]{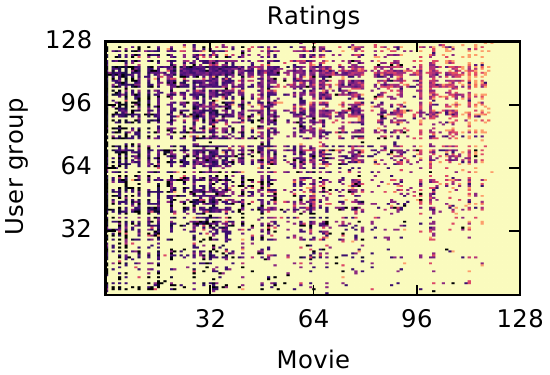}
  \includegraphics[width=2.2in]{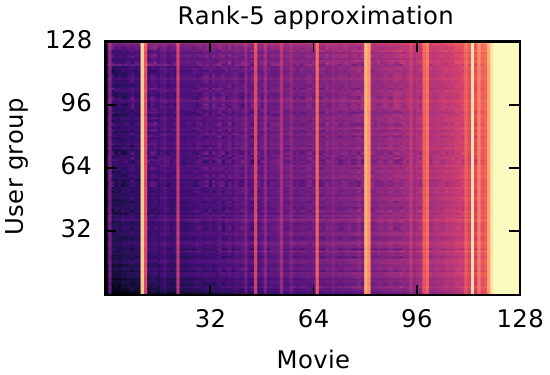}
  \includegraphics[width=2.2in]{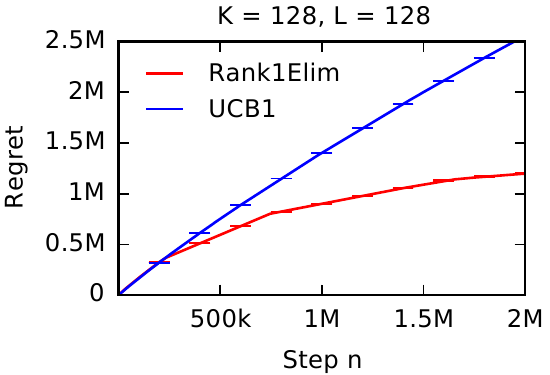} \\
  \hspace{0.35in} (a) \hspace{2in} (b) \hspace{2.05in} (c) \vspace{-0.05in}
  \caption{\textbf{a}. Ratings from the MovieLens dataset. The darker the color, the higher the rating. The rows and columns are ordered by their average ratings. The missing ratings are shown in yellow. \textbf{b}. Rank-$5$ approximation to the ratings. \textbf{c}. The $n$-step regret of $\bilinucb$ and $\ucb$ in up to $n = 2\text{M}$ steps.}
  \label{fig:movielens}
\end{figure*}

In our last experiment, we evaluate $\bilinucb$ on a recommendation problem. The goal is to identify the pair of a user group and movie that has the highest expected rating. We experiment with the \emph{MovieLens} dataset from February 2003 \cite{movielens}, where $6\text{k}$ users give $1\text{M}$ ratings to $4\text{k}$ movies.

Our learning problem is formulated as follows. We define a user group for every unique combination of gender, age group, and occupation in the MovieLens dataset. The total number of groups is $241$. For each user group and movie, we average the ratings of all users in that group that rated that movie, and learn a low-rank approximation to the underlying rating matrix by a state-of-the-art algorithm \cite{keshavan10matrix}. The algorithm automatically detects the rank of the matrix to be $5$. We randomly choose $K = 128$ user groups and $L = 128$ movies. We report the average ratings of these user groups and movies in \cref{fig:movielens}a, and the corresponding completed rating matrix in \cref{fig:movielens}b. The reward for choosing user group $i \in [K]$ and movie $j \in [L]$ is a categorical random variable over five-star ratings. We estimate its parameters based on the assumption that the ratings are normally distributed with a fixed variance, conditioned on the completed ratings. The expected rewards in this experiment are not rank $1$. Therefore, our model is misspecified and $\bilinucb$ has no guarantees on its performance.

Our results are reported in \cref{fig:movielens}c. We observe that the regret of $\bilinucb$ is concave in the number of steps $n$, and flattens. This indicates that $\bilinucb$ learns a near-optimal solution. This is possible because of the structure of our rating matrix. Although it is rank $5$, its first eigenvalue is an order of magnitude larger than the remaining four non-zero eigenvalues. This structure is not surprising because the ratings of items are often subject to significant \emph{user and item biases} \cite{koren09matrix}. Therefore, our rating matrix is nearly rank $1$, and $\bilinucb$ learns a good solution. Our theory cannot explain this result and we leave it for future work. Finally, we note that $\ucb$ explores throughout because our problem has more than $10\text{k}$ arms.

%!TEX root = Paper.tex

\section{Related Work}
\label{sec:related work}

Zhao \etal~\cite{zhao13interactive} proposed a bandit algorithm for low-rank matrix completion, where the posterior of latent item factors is approximated by its point estimate. This algorithm is not analyzed. Kawale \etal~\cite{kawale15efficient} proposed a Thompson sampling (TS) algorithm for low-rank matrix completion, where the posterior of low-rank matrices is approximated by particle filtering. A computationally-inefficient variant of the algorithm has $O((1 / \Delta^2) \log n)$ regret in rank-$1$ matrices. In contrast, note that $\bilinucb$ is computationally efficient and its $n$-step regret is $O((1 / \Delta) \log n)$.

The problem of learning to recommended in the bandit setting was studied in several recent papers. Valko \etal~\cite{valko14spectral} and Kocak \etal~\cite{kocak14spectral} proposed content-based recommendation algorithms, where the features of items are derived from a known similarity graph over the items. Gentile \etal~\cite{gentile14online} proposed an algorithm that clusters users based on their preferences, under the assumption that the features of items are known. Li \etal~\cite{li16collaborative} extended this algorithm to the clustering of items. Maillard \etal~\cite{maillard14latent} studied a multi-armed bandit problem where the arms are partitioned into latent groups. The problems in the last three papers  are a special form of low-rank matrix completion, where some rows are identical. In this work, we do not make any such assumptions, but our results are limited to rank $1$.

$\bilinucb$ is motivated by the structure of the position-based model \cite{craswell08experimental}. Lagree \etal~\cite{lagree16multipleplay} proposed a bandit algorithm for this model under the assumption that the examination probabilities of all positions are known. Online learning to rank in click models was studied in several recent papers \cite{kveton15cascading,combes15learning,kveton15combinatorial,katariya16dcm,li16contextual,zong16cascading}. In practice, the probability of clicking on an item depends on both the item and its position, and this work is a major step towards learning to rank from such heterogeneous effects.

%!TEX root = Paper.tex

\section{Conclusions}
\label{sec:conclusions}

In this work, we propose stochastic rank-$1$ bandits, a class of online learning problems where the goal is to learn the maximum entry of a rank-$1$ matrix. This problem is challenging because the reward is a product of latent random variables, which are not observed. We propose a practical algorithm for solving this problem, $\bilinucb$, and prove a gap-dependent upper bound on its regret. We also prove a nearly matching gap-dependent lower bound. Finally, we evaluate $\bilinucb$ empirically. In particular, we validate the scaling of its regret, compare it to baselines, and show that it learns high-quality solutions even when our modeling assumptions are mildly violated.

We conclude that $\bilinucb$ is a practical algorithm for finding the maximum entry of a stochastic rank-$1$ matrix. It is surprisingly competitive with various baselines (\cref{sec:comparison experiment}) and can be applied to higher-rank matrices (\cref{sec:movielens experiment}). On the other hand, we show that $\bilinucb$ can be suboptimal on relatively simple problems (\cref{sec:discussion}). We plan to address this issue in our future work. We note that our results can be generalized to other reward models, such as $\rnd{u}_t(i) \rnd{v}_t(j) \sim \mathcal{N}(\bar{u}(i) \bar{v}(j), \sigma)$ for $\sigma > 0$.

%!TEX root = Paper.tex

\subsubsection*{Acknowledgments}

This work was partially supported by NSERC and by the Alberta Innovates Technology Futures through the Alberta  Machine Intelligence Institute (AMII).

\bibliographystyle{plain}
\bibliography{References}

%!TEX root = Paper.tex

\clearpage
\onecolumn
\appendix

\section{Upper Bound}
\label{sec:lemmas}

\begin{lemma}
\label{lem:bad events} Let $\ccE$ be defined as in the proof of \cref{thm:upper bound}. Then
\begin{align*}
  P(\ccE) \leq
  \frac{2 (K + L)}{n}\,.
\end{align*}
\end{lemma}
\begin{proof}
Let $\cE_\ell = \cE^\rowvar_\ell \cap \cE^\colvar_\ell$. Then from the definition of $\ccE$,
\begin{align*}
  \ccE =
  \overline{\cE_0} \cup
  (\overline{\cE_1} \cap \cE_0) \cup \ldots \cup
  (\overline{\cE_{n - 1}} \cap \cE_{n - 2} \cap \ldots \cap \cE_0)\,,
\end{align*}
and from the definition of $\cE_\ell$,
\begin{align*}
  \overline{\cE_\ell} \cap \cE_{\ell - 1} \cap \ldots \cap \cE_0 =
  (\overline{\cE^\rowvar_\ell} \cap \cE_{\ell - 1} \cap \ldots \cap \cE_0) \cup
  (\overline{\cE^\colvar_\ell} \cap \cE_{\ell - 1} \cap \ldots \cap \cE_0)\,.
\end{align*}
It follows that the probability of event $\ccE$ is bounded as
\begin{align*}
  P(\ccE)
  & \leq \sum_{\ell = 0}^{n - 1} P(\overline{\cE^\rowvar_\ell}, \
  \cE^\rowvar_0, \ \dots, \ \cE^\rowvar_{\ell - 1}, \ \cE^\colvar_0, \ \dots, \ \cE^\colvar_{\ell - 1}) +
  \sum_{\ell = 0}^{n - 1} P(\overline{\cE^\colvar_\ell}, \
  \cE^\rowvar_0, \ \dots, \ \cE^\rowvar_{\ell - 1}, \ \cE^\colvar_0, \ \dots, \ \cE^\colvar_{\ell - 1}) \\
  & \leq \sum_{\ell = 0}^{n - 1} P(\overline{\cE^\rowvar_\ell}, \ \cE^\colvar_0, \ \dots, \ \cE^\colvar_{\ell - 1}) +
  \sum_{\ell = 0}^{n - 1} P(\overline{\cE^\colvar_\ell}, \ \cE^\rowvar_0, \ \dots, \ \cE^\rowvar_{\ell - 1})\,.  
\end{align*}
From the definition of $\overline{\cE^\rowvar_\ell}$, it follows that
\begin{align*}
  P(\overline{\cE^\rowvar_\ell}, \ \cE^\colvar_0, \ \dots, \ \cE^\colvar_{\ell - 1})
  \leq {} & P(\exists i \in \rnd{I}_\ell \text{ s.t. }
  \bar{\rnd{u}}_\ell(i) \notin [\rnd{L}^\rowvar_\ell(i), \rnd{U}^\rowvar_\ell(i)]) + {} \\
  & P(\exists i \in \rnd{I}_\ell \text{ s.t. }
  \bar{\rnd{u}}_\ell(i) < \mu \bar{u}(i), \ \cE^\colvar_0, \ \dots, \ \cE^\colvar_{\ell - 1})\,.
\end{align*}
Now we bound the probability of the above two events. The probability $P(\overline{\cE^\colvar_\ell}, \ \cE^\rowvar_0, \ \dots, \ \cE^\rowvar_{\ell - 1})$ can be bounded similarly and we omit this proof.

\vspace{0.1in}

\textbf{Event $1$:}
$\exists i \in \rnd{I}_\ell \text{ s.t. } \bar{\rnd{u}}_\ell(i) \notin [\rnd{L}^\rowvar_\ell(i), \rnd{U}^\rowvar_\ell(i)]$

Fix any $i \in \rnd{I}_\ell$. Let $\rnd{c}_k$ be the $k$-th observation of row $i$ in the row exploration stage of $\bilinucb$ and $\ell(k)$ be the index of that stage. Then
\begin{align*}
  \left(\rnd{c}_k - \bar{u}(i) \sum_{j = 1}^L \frac{\bar{v}(\rnd{h}^\colvar_{\ell(k)}(j))}{L}\right)_{k = 1}^n
\end{align*}
is a martingale difference sequence with respect to history $\rnd{h}^\colvar_0, \dots, \rnd{h}^\colvar_{\ell(k)}$ in step $k$. This follows from the observation that
\begin{align*}
  \condE{\rnd{c}_k}{\rnd{h}^\colvar_0, \dots, \rnd{h}^\colvar_{\ell(k)}} =
  \bar{u}(i) \sum_{j = 1}^L \frac{\bar{v}(\rnd{h}^\colvar_{\ell(k)}(j))}{L}\,,
\end{align*}
because column $j \in [L]$ in stage $\ell(k)$ is chosen randomly and then mapped to at least as rewarding column $\rnd{h}^\colvar_{\ell(k)}(j)$. By the definition of our sequence and from the Azuma-Hoeffding inequality (Remark 2.2.1 of Raginsky and Sason \cite{raginsky12concentration}),
\begin{align*}
  P(\bar{\rnd{u}}_\ell(i) \notin [\rnd{L}^\rowvar_\ell(i), \rnd{U}^\rowvar_\ell(i)])
  & = P\left(\Bigg|\frac{1}{n_\ell} \sum_{j = 1}^L \rnd{C}^\rowvar_\ell(i, j) -
  \bar{\rnd{u}}_\ell(i)\Bigg| > \sqrt{\frac{\log n}{n_\ell}}\right) \\
  & = P\left(\Bigg|\sum_{k = 1}^{n_\ell} \Bigg[\rnd{c}_k -
  \bar{u}(i) \sum_{j = 1}^L \frac{\bar{v}(\rnd{h}^\colvar_{\ell(k)}(j))}{L}\Bigg]\Bigg| > \sqrt{n_\ell \log n}\right) \\
  & \leq 2 \exp[-2 \log n] \\
  & = 2 n^{-2}
\end{align*}
for any stage $\ell$. By the union bound,
\begin{align*}
  P(\exists i \in \rnd{I}_\ell \text{ s.t. }
  \bar{\rnd{u}}_\ell(i) \notin [\rnd{L}^\rowvar_\ell(i), \rnd{U}^\rowvar_\ell(i)]) \leq
  2 K n^{-2}
\end{align*}
for any stage $\ell$.

\vspace{0.1in}

\textbf{Event $2$:}
$\exists i \in \rnd{I}_\ell \text{ s.t. } \bar{\rnd{u}}_\ell(i) < \mu \bar{u}(i), \ \cE^\colvar_0, \ \dots, \ \cE^\colvar_{\ell - 1}$

We claim that this event cannot happen. Fix any $i \in \rnd{I}_\ell$. When $\ell = 0$, we get that $\bar{\rnd{u}}_0(i) = \bar{u}(i) (1 / L) \sum_{j = 1}^L \bar{v}(j) \geq \mu \bar{u}(i)$ from the definitions of $\bar{\rnd{u}}_0(i)$ and $\mu$, and event $2$ obviously does not happen. When $\ell > 0$ and events $\cE^\colvar_0, \dots, \cE^\colvar_{\ell - 1}$ happen, any eliminated column $j$ up to stage $\ell$ is substituted with column $j'$ such that $\bar{v}(j') \geq \bar{v}(j)$, by the design of $\bilinucb$. From this fact and the definition of $\bar{\rnd{u}}_\ell(i)$, $\bar{\rnd{u}}_\ell(i) \geq \mu \bar{u}(i)$. Therefore, event $2$ does not happen when $\ell > 0$.

\vspace{0.1in}

\noindent \textbf{Total probability}

\noindent Finally, we sum all probabilities up and get that
\begin{align*}
  P(\ccE) \leq
  n \left(\frac{2 K}{n^2}\right) + n \left(\frac{2 L}{n^2}\right) \leq
  \frac{2 (K + L)}{n}\,.
\end{align*}
This concludes our proof.
\end{proof}

\begin{lemma}
\label{lem:elimination} Let event $\cE$ happen and $m$ be the first stage where $\tilde{\Delta}_m < \mu \Delta^\rowvar_i / 2$. Then row $i$ is guaranteed to be eliminated by the end of stage $m$. Moreover, let $m$ be the first stage where $\tilde{\Delta}_m < \mu \Delta^\colvar_j / 2$. Then column $j$ is guaranteed to be eliminated by the end of stage $m$.
\end{lemma}
\begin{proof}
We only prove the first claim. The other claim is proved analogously.

Before we start, note that by the design of $\bilinucb$ and from the definition of $m$,
\begin{align}
  \tilde{\Delta}_m =
  2^{- m} <
  \frac{\mu \Delta^\rowvar_i}{2} \leq
  2^{- (m - 1)} =
  \tilde{\Delta}_{m - 1}\,.
  \label{eq:gap and tilde gap}
\end{align}
By the design of our confidence intervals,
\begin{align*}
  \frac{1}{n_m} \sum_{j = 1}^K \rnd{C}^\rowvar_m(i, j) + \sqrt{\frac{\log n}{n_m}}
  & \stackrel{\text{(a)}}{\leq} \bar{\rnd{u}}_m(i) + 2 \sqrt{\frac{\log n}{n_m}} \\
  & = \bar{\rnd{u}}_m(i) + 4 \sqrt{\frac{\log n}{n_m}} - 2 \sqrt{\frac{\log n}{n_m}} \\
  & \stackrel{\text{(b)}}{\leq} \bar{\rnd{u}}_m(i) + 2 \tilde{\Delta}_m - 2 \sqrt{\frac{\log n}{n_m}} \\
  & \stackrel{\text{(c)}}{\leq} \bar{\rnd{u}}_m(i) + \mu \Delta^\rowvar_i - 2 \sqrt{\frac{\log n}{n_m}} \\
  & = \bar{\rnd{u}}_m(i^\ast) + \mu \Delta^\rowvar_i -
  [\bar{\rnd{u}}_m(i^\ast) - \bar{\rnd{u}}_m(i)] - 2 \sqrt{\frac{\log n}{n_m}}\,,
\end{align*}
where inequality (a) is from $\rnd{L}^\rowvar_m(i) \leq \bar{\rnd{u}}_m(i)$, inequality (b) is from $n_m \geq 4 \tilde{\Delta}_m^{-2} \log n$, and inequality (c) is by \eqref{eq:gap and tilde gap}. Now note that
\begin{align*}
  \bar{\rnd{u}}_m(i^\ast) - \bar{\rnd{u}}_m(i) =
  q (\bar{u}(i^\ast) - \bar{u}(i)) \geq
  \mu \Delta^\rowvar_i
\end{align*}
for some $q \in [0, 1]$. The equality holds because $\bar{\rnd{u}}_m(i^\ast)$ and $\bar{\rnd{u}}_m(i)$ are estimated from the same sets of random columns. The inequality follows from the fact that events $\cE^\colvar_0, \dots, \cE^\colvar_{m - 1}$ happen. The events imply that any eliminated column $j$ up to stage $m$ is substituted with column $j'$ such that $\bar{v}(j') \geq \bar{v}(j)$, and thus $q \geq \mu$. From the above inequality, we get that
\begin{align*}
  \bar{\rnd{u}}_m(i^\ast) + \mu \Delta^\rowvar_i -
  [\bar{\rnd{u}}_m(i^\ast) - \bar{\rnd{u}}_m(i)] - 2 \sqrt{\frac{\log n}{n_m}} \leq
  \bar{\rnd{u}}_m(i^\ast) - 2 \sqrt{\frac{\log n}{n_m}}\,.
\end{align*}
Finally,
\begin{align*}
  \bar{\rnd{u}}_m(i^\ast) - 2 \sqrt{\frac{\log n}{n_m}}
  & \stackrel{\text{(a)}}{\leq} \frac{1}{n_m} \sum_{j = 1}^K \rnd{C}^\rowvar_m(i^\ast, j) - \sqrt{\frac{\log n}{n_m}} \\
  & \stackrel{\text{(b)}}{\leq} \frac{1}{n_m} \sum_{j = 1}^K \rnd{C}^\rowvar_m(\rnd{i}_m, j) - \sqrt{\frac{\log n}{n_m}}\,,
\end{align*}
where inequality (a) follows from $\bar{\rnd{u}}_m(i^\ast) \leq \rnd{U}^\rowvar_m(i^\ast)$ and inequality (b) follows from $\rnd{L}^\rowvar_m(i^\ast) \leq \rnd{L}^\rowvar_m(\rnd{i}_m)$, since $i^\ast \in \rnd{I}_m$ and $\rnd{i}_m = \argmax_{i \in \rnd{I}_m} \rnd{L}^\rowvar_m(i)$. Now we chain all inequalities and get our final claim.
\end{proof}

\begin{lemma}
\label{lem:row regret} The expected cumulative regret due to exploring any row $i \in [K]$ and any column $j \in [L]$ is bounded as
\begin{align*}
  \E{\sum_{\ell = 0}^{n - 1} \condE{\rnd{R}^\rowvar_\ell(i)}{\cH_\ell} \I{\cF_\ell}}
  & \leq \frac{384}{\mu^2 \bar{\Delta}^\rowvar_i} \log n + 1\,, \\
  \E{\sum_{\ell = 0}^{n - 1} \condE{\rnd{R}^\colvar_\ell(j)}{\cH_\ell} \I{\cF_\ell}}
  & \leq \frac{384}{\mu^2 \bar{\Delta}^\colvar_j} \log n + 1\,.
\end{align*}
\end{lemma}
\begin{proof}
We only prove the first claim. The other claim is proved analogously. This proof has two parts. In the first part, we assume that row $i$ is suboptimal, $\Delta^\rowvar_i > 0$. In the second part, we assume that row $i$ is optimal, $\Delta^\rowvar_i = 0$.

\vspace{0.1in}

\textbf{Row $i$ is suboptimal}

Let row $i$ be suboptimal and $m$ be the first stage where $\tilde{\Delta}_m < \mu \Delta^\rowvar_i / 2$. Then row $i$ is guaranteed to be eliminated by the end of stage $m$ (\cref{lem:elimination}), and thus
\begin{align*}
  \E{\sum_{\ell = 0}^{n - 1} \condE{\rnd{R}^\rowvar_\ell(i)}{\cH_\ell} \I{\cF_\ell}} \leq
  \E{\sum_{\ell = 0}^m \condE{\rnd{R}^\rowvar_\ell(i)}{\cH_\ell} \I{\cF_\ell}}\,.
\end{align*}
By \cref{lem:componentwise regret}, the expected regret of choosing row $i$ in stage $\ell$ can be bounded from above as
\begin{align*}
  \condE{\rnd{R}^\rowvar_\ell(i)}{\cH_\ell} \I{\cF_\ell} \leq
  (\Delta^\rowvar_i + \max_{j \in \rnd{J}_\ell} \Delta^\colvar_j) (n_\ell - n_{\ell - 1})\,,
\end{align*}
where $\max_{j \in \rnd{J}_\ell} \Delta^\colvar_j$ is the maximum column gap in stage $\ell$, $n_\ell$ is the number of steps by the end of stage $\ell$, and $n_{-1} = 0$. From the definition of $\mathcal{F}_\ell$ and $\tilde{\Delta}_\ell$, if column $j$ is not eliminated before stage $\ell$, we have that
\begin{align*}
  \Delta^\colvar_j \leq
  \frac{2 \tilde{\Delta}_{\ell - 1}}{\mu} =
  \frac{2 \cdot 2^{m - \ell + 1} \tilde{\Delta}_m}{\mu} <
  2^{m - \ell + 1} \Delta^\rowvar_i\,.
\end{align*}
From the above inequalities and the definition of $n_\ell$, it follows that
\begin{align*}
  \E{\sum_{\ell = 0}^m \condE{\rnd{R}^\rowvar_\ell(i)}{\cH_\ell} \I{\cF_\ell}}
  & \leq \sum_{\ell = 0}^m (\Delta^\rowvar_i + \max_{j \in \rnd{J}_\ell} \Delta^\colvar_j) (n_\ell - n_{\ell - 1}) \\
  & \leq \sum_{\ell = 0}^m (\Delta^\rowvar_i + 2^{m - \ell + 1} \Delta^\rowvar_i) (n_\ell - n_{\ell - 1}) \\
  & \leq \Delta^\rowvar_i \left(n_m + \sum_{\ell = 0}^m 2^{m - \ell + 1} n_\ell\right) \\
  & \leq \Delta^\rowvar_i \left(2^{2 m + 2} \log n + 1 +
  \sum_{\ell = 0}^m 2^{m - \ell + 1} (2^{2 \ell + 2} \log n + 1)\right) \\
  & = \Delta^\rowvar_i \left(2^{2 m + 2} \log n + 1 +
  \sum_{\ell = 0}^m 2^{m + \ell + 3} \log n + \sum_{\ell = 0}^m 2^{m - \ell + 1}\right) \\
  & \leq \Delta^\rowvar_i (5 \cdot 2^{2 m + 2} \log n + 2^{m + 2}) + 1 \\
  & \leq 6 \cdot 2^4 \cdot 2^{2 m - 2} \Delta^\rowvar_i \log n + 1\,,
\end{align*}
where the last inequality follows from $\log n \geq 1$ for $n \geq 3$. From the definition of $\tilde{\Delta}_{m - 1}$ in \eqref{eq:gap and tilde gap}, we have that
\begin{align*}
  2^{m - 1} =
  \frac{1}{\tilde{\Delta}_{m - 1}} \leq
  \frac{2}{\mu \Delta^\rowvar_i}\,.
\end{align*}
Now we chain all above inequalities and get that
\begin{align*}
  \E{\sum_{\ell = 0}^{n - 1} \condE{\rnd{R}^\rowvar_\ell(i)}{\cH_\ell} \I{\cF_\ell}} \leq
  6 \cdot 2^4 \cdot 2^{2 m - 2} \Delta^\rowvar_i \log n + 1 \leq
  \frac{384}{\mu^2 \Delta^\rowvar_i} \log n + 1\,.
\end{align*}
This concludes the first part of our proof.

\vspace{0.1in}

\textbf{Row $i$ is optimal}

Let row $i$ be optimal and $m$ be the first stage where $\tilde{\Delta}_m < \mu \Delta^\colvar_{\min} / 2$. Then similarly to the first part of the analysis,
\begin{align*}
  \E{\sum_{\ell = 0}^{n - 1} \condE{\rnd{R}^\rowvar_\ell(i)}{\cH_\ell} \I{\cF_\ell}} \leq
  \sum_{\ell = 0}^m (\max_{j \in \rnd{J}_\ell} \Delta^\colvar_j) (n_\ell - n_{\ell - 1}) \leq
  \frac{384}{\mu^2 \Delta^\colvar_{\min}} \log n + 1\,.
\end{align*}
This concludes our proof.
\end{proof}

\begin{lemma}
\label{lem:componentwise regret} Let $\rnd{u} \sim P_\rowvar$ and $\rnd{v} \sim P_\colvar$ be drawn independently. Then the expected regret of choosing any row $i \in [K]$ and column $j \in [L]$ is bounded from above as
\begin{align*}
  \E{\rnd{u}(i^\ast) \rnd{v}(j^\ast) - \rnd{u}(i) \rnd{v}(j)} \leq
  \Delta^\rowvar_i + \Delta^\colvar_j\,.
\end{align*}
\end{lemma}
\begin{proof}
Note that for any $x, y, x^\ast, y^\ast \in [0, 1]$,
\begin{align*}
  x^\ast y^\ast - x y =
  x^\ast y^\ast - x y^\ast + x y^\ast - x y =
  y^\ast (x^\ast - x) + x (y^\ast - y) \leq
  (x^\ast - x) + (y^\ast - y)\,.
\end{align*}
By the independence of the entries of $\rnd{u}$ and $\rnd{v}$, and from the above inequality,
\begin{align*}
  \E{\rnd{u}(i^\ast) \rnd{v}(j^\ast) - \rnd{u}(i) \rnd{v}(j)} =
  \bar{u}(i^\ast) \bar{v}(j^\ast) - \bar{u}(i) \bar{v}(j) \leq
  (\bar{u}(i^\ast) - \bar{u}(i)) + (\bar{v}(j^\ast) - \bar{v}(j))\,.
\end{align*}
This concludes our proof.
\end{proof}

\section{Lower Bound}
\label{sec:LBdetails}

In this section we present the missing details of the proof of \cref{thm:lower bound}. Recall that we need to bound from below the value of $f(\bar{u},\bar{v})$ where
\begin{align*}
f(\bar{u}, \bar{v})
= {} & \inf_{c \in [0,\infty)^{K \times L}} \sum_{i = 1}^K \sum_{j = 1}^L
(\bar{u}(i^\ast) \bar{v}(j^\ast) - \bar{u}(i) \bar{v}(j)) c_{i, j} \\
& \text{s.t. } \forall (\bar{u}', \bar{v}') \in B(\bar{u}, 
\bar{v}): \\
& \phantom{\text{s.t.}} \quad \!
\sum_{i = 1}^K \sum_{j = 1}^L d(\bar{u}(i) \bar{v}(j), \bar{u}'(i) \bar{v}'(j)) 
c_{i, j} \geq 1
\end{align*}
and %\todoc{Do we need $\bar{u}(i^*), \bar{v}(j^*)<1$?}
\[
  B(\bar{u},\bar{v})= \{ (\bar{u}',\bar{v}') \in [0,1]^K\times [0,1]^L 
 : 
  \bar{u}(i^\ast) = \bar{u}'(i^\ast),\, 
  \bar{v}(j^\ast) = \bar{v}'(j^\ast)  ,\, 
  w^\ast(\bar{u},\bar{v})<w^\ast(\bar{u}',\bar{v}')  \}\,.
\] 
Without loss of generality, we assume that the optimal action 
in the original model $(\bar{u}, \bar{v})$ is $(i^*,j^*)=(1,1)$.
Moreover, we consider a class of \emph{identifiable} bandit models, meaning that we assume that
\[
\forall (i,i',j,j')\in [0,1]^{2K}\times[0,1]^{2L}, \quad 
(i,j)\neq (i',j') \implies 0<d(\bar{u}(i)\bar{v}(j),\bar{u}(i')\bar{v}(j'))<+\infty.
\]
This implies in particular that $\bar{u}(i^*)\bar{v}(j^*)$ must be less than $1$. An intuitive justification of this assumption is the following. Remark that 
for the Bernoulli problem we consider here, if the mean of the best arm is exactly $1$, the rewards from optimal pulls 
are always $1$ so that the empirical average is always exactly $1$ and as we cap the UCBs to $1$, the optimal arm 
is always a candidate to the next pull, which leads to constant regret. 
Also note that by our assumption, the optimal action is unique. 
%\todoc{Is this essential? Does the same bound hold otherwise?}
To get a lower bound, we consider the same optimization problem as above, but replace $B$ with its subset. Clearly, this can only decrease the optimal value.

Concretely, we consider only those models in $B(\bar{u},\bar{v})$ where only one parameter changes at a time. Let
\begin{align*}
B_\rowvar(\bar{u},\bar{v}) &= \{ (\bar{u}',\bar{v})  : 
											\bar{u}'\in [0,1]^K,\,
											\exists i_0\in \{2,\dots,K\}, \epsilon\in [0,1] \text{ s.t. }
[\forall i \neq i_0: \bar{u}'(i) = \bar{u}(i)] \text{ and } \bar{u}'(i_0) = \bar{u}(1) + \epsilon
											\}\,,\\
B_\colvar(\bar{u},\bar{v}) &= \{ (\bar{u},\bar{v}')  : 
											\bar{v}'\in [0,1]^L,\,
											\exists j_0\in \{2,\dots,L\}, \epsilon\in [0,1] \text{ s.t. }
[\forall j \neq j_0: \bar{v}'(j) = \bar{v}(j)] \text{ and } \bar{v}'(j_0) = \bar{v}(1) + \epsilon
											\}\,.
\end{align*}
Let $f'(\bar{u}, \bar{v})$ be the optimal value of the above optimization problem when $B(\bar{u}, \bar{v})$ is replaced by $B_\rowvar(\bar{u}, \bar{v}) \cup B_\colvar(\bar{u}, \bar{v}) \subset B(\bar{u}, \bar{v})$. Now suppose that $(\bar{u}', \bar{v}') \in B_\rowvar(\bar{u}, \bar{v})$ and $i_0 = 2$. Then, for any $i \neq 2$ and $j \in [L]$, $d(\bar{u}(i) \bar{v}(j), \bar{u}'(i) \bar{v}'(j)) = 0$; and for $i = 2$ and any $j \in [L]$, $d(\bar{u}(i) \bar{v}(j), \bar{u}'(i) \bar{v}'(j)) = d(\bar{u}(2) \bar{v}(j), (\bar{u}(1) + \epsilon) \bar{v}(j))$. Hence, 
\begin{align*}
%\MoveEqLeft
\sum_{i = 1}^K \sum_{j = 1}^L d(\bar{u}(i) \bar{v}(j), \bar{u}'(i) \bar{v}'(j)) 
=\sum_{j = 1}^L d(\bar{u}(2) \bar{v}(j), (\bar{u}(1)+\epsilon) \bar{v}(j))\,.
\end{align*}
Reasoning similarly for $B_\colvar(\bar{u},\bar{v})$, we see that $f'(\bar{u},\bar{v})$ satisfies
\begin{align*}
f'(\bar{u}, \bar{v})
= {} & \inf_{c \in [0,\infty)^{K \times L}} \sum_{i = 1}^K \sum_{j = 1}^L
(\bar{u}(i^\ast) \bar{v}(j^\ast) - \bar{u}(i) \bar{v}(j)) c_{i, j} \\
& \text{s.t.} \quad \forall  \epsilon_\colvar\in (0,1-\bar{v}(1)], \epsilon_\rowvar\in (0,1-\bar{u}(1)]\,\\
& \phantom{\text{s.t} \,}  \quad \forall j\neq 1, 
		\sum_{i = 1}^K  d(\bar{u}(i) \bar{v}(j), \bar{u}(i) (\bar{v}(1)+\epsilon_\colvar))c_{i, j} \geq 1\\
&  \phantom{\text{s.t} \,} \quad \forall i\neq 1, \, 
		\sum_{j= 1}^L  d(\bar{u}(i) \bar{v}(j), (\bar{u}(1)+\epsilon_\rowvar) \bar{v}(j))c_{i, j} \geq 1.
\end{align*}
Clearly, the smaller the coefficients of $c_{i, j}$ in the constraints, the tighter the constraints. We obtain the smallest coefficients when $\epsilon_\colvar, \epsilon_\rowvar \to 0$.
By continuity, we get 
\begin{align*}
f'(\bar{u}, \bar{v})
= {} & \inf_{c \in [0,\infty)^{K \times L}}
 \sum_{i = 1}^K \sum_{j = 1}^L (\bar{u}(i^\ast) \bar{v}(j^\ast) - \bar{u}(i) \bar{v}(j)) c_{i, j} \\
& \text{s.t.} \quad \forall j\neq 1, \,
\sum_{i = 1}^K  d(\bar{u}(i) \bar{v}(j), \bar{u}(i) \bar{v}(1))c_{i, j} \geq 1\\
&  \phantom{\text{s.t} \,} \quad \forall i\neq 1, \, 
\sum_{j= 1}^L  d(\bar{u}(i)  \bar{v}(j), \bar{u}(1) \bar{v}(j))c_{i, j} \geq 1.
\end{align*}

Let
\begin{align*}
c_{i, j} =
\begin{cases}
1 / d(\bar{u}(i) \bar{v}(1), \bar{u}(1) \bar{v}(1))\,, & j = 
1 \text{ and } i>1\,; \\
1 / d(\bar{u}(1) \bar{v}(j), \bar{u}(1) \bar{v}(1))\,, & i = 
1 \text{ and } j>1\,; \\
0\,, & \text{otherwise}.
\end{cases}
\end{align*}
We claim that $(c_{i, j})$ is an optimal solution for the problem defining $f'$.

First, we show that $(c_{i, j})$ is feasible.
Let $i\ne 1$. Then
$
\sum_{j= 1}^L  d(\bar{u}(i)  \bar{v}(j), \bar{u}(1) \bar{v}(j))c_{i, j}
= d(\bar{u}(i)  \bar{v}(1), \bar{u}(1) \bar{v}(1)) c_{i,1} = 1
$.
Similarly, we can verify the other constraint, too, showing that $(c_{i, j})$ is indeed feasible.

Now, it remains to show that the proposed solution is indeed optimal. 
We prove this by
contradiction, following the ideas of \cite{combes15learning}. We 
suppose that there exists a 
solution $c$ of the optimization problem such that $c_{i_0,j_0}>0$ for $i_0\neq 
1$ and $j_0\neq 1$. Then, we prove that it is possible to find another feasible 
solution $c'$ but with an objective lower than that obtained with $c$, 
contradicting the assumption of optimality of $c$.

We define $c'$ as follows, redistributing the mass of $c_{i_0,j_0}$ on the 
first row and the first column:
\begin{align*}
c'_{i,j} =
\begin{cases}
0\,, & i=i_0 \text{ and } j=j_0\,;\\
c_{i_0,1} +c_{i_0,j_0} 
\frac{\displaystyle 
d(\bar{u}(i_0)\bar{v}(j_0),\bar{u}(1)\bar{v}(j_0))}{\displaystyle 
d(\bar{u}(i_0)\bar{v}(1),\bar{u}(1)\bar{v}(1))}\,,
& i=i_0 \text{ and } j=1\,;\\
& \\
c_{1,{j_0}} +c_{i_0,j_0} 
\frac{\displaystyle 
d(\bar{u}(i_0)\bar{v}(j_0),\bar{u}(i_0)\bar{v}(1))}{\displaystyle 
d(\bar{u}(1)\bar{v}(j_0),\bar{u}(1)\bar{v}(1))}\,,
& i=1 \text{ and } j=j_0\,;\\
c_{i,j}\,, & \mathrm{otherwise}.		
\end{cases}
\end{align*}

It is easily verified that if $c$ satisfies the constraints, then so 
does $c'$ because the missing mass of $c_{i_0,j_0}$ is simply redistributed on  
$c'_{i_0,1}$ and $c'_{1,j_0}$. For example, for $i=i_0$
we have 
\begin{align*}
&\sum_{j=1}^L d(\bar{u}(i_0)\bar{v}(j),\bar{u}(1)\bar{v}(j))c'_{i_0,j}  - 
\sum_{j=1}^L 
d(\bar{u}(i_0)\bar{v}(j),\bar{u}(1)\bar{v}(j))c_{i_0,j}  \\ 
& = d(\bar{u}(i_0)\bar{v}(1),\bar{u}(1)\bar{v}(1))c_{i_0,j_0} 
\frac{d(\bar{u}(i_0)\bar{v}(j_0),\bar{u}(1)\bar{v}(j_0))}{d(\bar{u}(i_0)\bar{v}(1),\bar{u}(1)\bar{v}(1))}
- c_{i_0,j_0}d(\bar{u}(i_0)\bar{v}(j_0),\bar{u}(1)\bar{v}(j_0))
%\frac{d(\bar{u}(i_0)\bar{v}(j_0),\bar{u}(1)\bar{v}(j_0))}{d(\bar{u}(1)\bar{v}(j_0),\bar{u}(1)\bar{v}(1))}
\\
& =0
\end{align*}
while for $i\not\in\{ 1,i_0\}$, $c'_{i,j} = c_{i,j}$, so 
$\sum_{j=1}^L d(\bar{u}(i)\bar{v}(j),\bar{u}(1)\bar{v}(j))c'_{i,j} 
=\sum_{j=1}^L d(\bar{u}(i)\bar{v}(j),\bar{u}(1)\bar{v}(j))c_{i,j} $.

Now, we prove that the objective function is lower for $c'$ than for 
$c$ by showing that the difference between them is negative:
\begin{align*}
\Delta &\doteq \sum_{i=1}^K\sum_{j=1}^L (\bar{u}(1)\bar{v}(1)-\bar{u}(i)\bar{v}(j))c'_{i,j} - 
\sum_{i=1}^K\sum_{j=1}^L (\bar{u}(1)\bar{v}(1)-\bar{u}(i)\bar{v}(j))c_{i,j}\\ 
& =\quad\,\,\,\, c_{i_0,j_0}\,(\bar{u}(1)\bar{v}(1) -\bar{u}(i_0)\bar{v}(1))
\frac{d(\bar{u}(i_0)\bar{v}(j_0),\bar{u}(1)\bar{v}(j_0))}
{d(\bar{u}(i_0)\bar{v}(1),\bar{u}(1)\bar{v}(1))} \nonumber \\
&\qquad + c_{i_0,j_0}(\bar{u}(1)\bar{v}(1) -\bar{u}(1)\bar{v}(j_0)) 
\frac{d(\bar{u}(i_0)\bar{v}(j_0),\bar{u}(i_0)\bar{v}(1))}
{d(\bar{u}(1)\bar{v}(j_0),\bar{u}(1)\bar{v}(1))} \nonumber \\
&\qquad- c_{i_0,j_0}(\bar{u}(1)\bar{v}(1) -\bar{u}(i_0)\bar{v}(j_0)) \\
& =\quad\,\,\,\, c_{i_0,j_0}\,\Big\{ (\bar{u}(1) -\bar{u}(i_0))\bar{v}(1)
\frac{d(\bar{u}(i_0)\bar{v}(j_0),\bar{u}(1)\bar{v}(j_0))}
{d(\bar{u}(i_0)\bar{v}(1),\bar{u}(1)\bar{v}(1))} \nonumber \\
&\qquad\qquad\quad + (\bar{v}(1) -\bar{v}(j_0)) \bar{u}(1)
\frac{d(\bar{u}(i_0)\bar{v}(j_0),\bar{u}(i_0)\bar{v}(1))}
{d(\bar{u}(1)\bar{v}(j_0),\bar{u}(1)\bar{v}(1))} \nonumber \\
&\qquad\qquad\quad- (\bar{u}(1)\bar{v}(1) -\bar{u}(i_0)\bar{v}(j_0)) \Big\}\\
\label{eq:diff}
\end{align*}
Writing
\[
\bar{u}(1)\bar{v}(1) - \bar{u}(i_0)\bar{v}(j_0) = 
(\bar{u}(1)-\bar{u}(i_0))\bar{v}(j_0)
+ 
(\bar{v}(1)-\bar{v}(j_0)) \bar{u}(1)
\]
we get
\begin{align}
\Delta
%&= (\bar{u}(1)\bar{v}(1)-\bar{u}(1)\bar{v}(j_0))c_{i_0,j_0} 
%\frac{\displaystyle 
%d(\bar{u}(i_0)\bar{v}(j_0),\bar{u}(1)\bar{v}(j_0))}{\displaystyle 
%d(\bar{u}(1)\bar{v}(j_0),\bar{u}(1)\bar{v}(1))}  \nonumber \\
%&+
%(\bar{u}(1)\bar{v}(1)-\bar{u}(i_0)\bar{v}(1))c_{i_0,j_0} 
%\frac{\displaystyle 
%d(\bar{u}(i_0)\bar{v}(j_0),\bar{u}(i_0)\bar{v}(1))}{\displaystyle 
%d(\bar{u}(i_0)\bar{v}(1),\bar{u}(1)\bar{v}(1))} \nonumber \\
%&- (\bar{u}(1)\bar{v}(1)-\bar{u}(i_0)\bar{v}(j_0)) c_{i_0,j_0} \nonumber\\
& =c_{i_0,j_0}(\bar{u}(1) -\bar{u}(i_0)) \left( \bar{v}(1) 
\frac{d(\bar{u}(i_0)\bar{v}(j_0),\bar{u}(1)\bar{v}(j_0))}
	   {d(\bar{u}(i_0)\bar{v}(1),\bar{u}(1)\bar{v}(1))}
 -\bar{v}(j_0) \right) 
\nonumber\\
&+ c_{i_0,j_0}(\bar{v}(1) -\bar{v}(j_0)) \left(\bar{u}(1) 
\frac{d(\bar{u}(i_0)\bar{v}(j_0),\bar{u}(i_0)\bar{v}(1))}
	   {d(\bar{u}(1)\bar{v}(j_0),\bar{u}(1)\bar{v}(1))} 
- \bar{u}(1)\right). 			
\nonumber
\end{align}

To finish the proof, it suffices to prove that both terms of the above sum are
negative. First, $\bar{u}(1) -\bar{u}(i_0),\bar{v}(1) -\bar{v}(j_0),c_{i_0,j_0}>0$,
hence it remains to consider the terms involving the ratios of KL divergences.
% For the first such term, note that the difference is increased if $-\bar{v}(j_0)$ is 
% replaced by $\bar{v}(1)$. Thus it remains to see whether the ratios of the KL divergences
% are below one. 
Note that both ratios take the form $\frac{d(\alpha p, \alpha q )}{d(p,q)}$ with $\alpha< 1$, but one
must be compared to $\alpha <1$ while the other can simply be compared to 1.
% \cref{lem:incrkappa} proven below shows that the function $\alpha \mapsto d(\alpha 
% p,\alpha q)$ is increasing on $(0,1)$, showing that
For the first such term, showing the negativity of the difference is equivalent to showing that 
for $\alpha=\bar{v}(j_0)/\bar{v}(1)<1$, 
\[
\frac{d(\alpha \bar{u}(i_0)\bar{v}(1),\alpha \bar{u}(1)\bar{v}(1))}
     {d(\bar{u}(i_0)\bar{v}(1),\bar{u}(1)\bar{v}(1))} < \alpha.
\]
Lemma~\ref{lem:incrkappa} below shows that for fixed $(p,q)\in (0,1)^2$, $f:\alpha \mapsto d(\alpha p,\alpha q)$ is convex, 
which proves the above inequality.
For the second term, it remains to see whether the ratio of the KL divergences
is below one. \cref{lem:incrkappa} proven below shows that the function $\alpha \mapsto d(\alpha 
p,\alpha q)$ is increasing on $(0,1)$, showing that
\[
% \frac{d(\bar{u}(i_0)\bar{v}(j_0),\bar{u}(1)\bar{v}(j_0))}
% 	   {d(\bar{u}(i_0)\bar{v}(1),\bar{u}(1)\bar{v}(1))}\,,\qquad
\frac{d(\bar{u}(i_0)\bar{v}(j_0),\bar{u}(i_0)\bar{v}(1))}
	   {d(\bar{u}(1)\bar{v}(j_0),\bar{u}(1)\bar{v}(1))} <1\,.
\]
Thus, the proof is finished once we prove \cref{lem:incrkappa}.
\if0
because $d(\bar{u}(i_0)\bar{v}(j_0),\bar{u}(1)\bar{v}(j_0))< d(\bar{u}(i_0)\bar{v}(1),\bar{u}(1)\bar{v}(1))$ 
for fixed $\bar{u}(i_0) < \bar{u}(1)$.
Moreover, the second term is also negative because of the property
proven in \cref{lem:incrkappa} for fixed $p < q$. 
\fi

%\subsection{Technical Lemmas}

\begin{lemma}\label{lem:incrkappa}
	Let $p,q$ be any fixed real numbers in $(0,1)$. The function 
	$f:\alpha 
	\mapsto d(\alpha p,\alpha q)$  is convex and increasing on $(0,1)$. As 
	a consequence, for any $\alpha<1$, $d(\alpha p, \alpha q) < d(p,q)$.
\end{lemma}

\begin{proof}
	We first re-parametrize our problem into polar coordinates $(r,\theta)$ :
	$$\begin{cases}
	p & =  r\cos \theta \\
	q & =  r\sin \theta 
	\end{cases}
	$$
	In order to prove the statement of the lemma, it now suffices to prove 
	that $f_\theta: r \mapsto d(r\sin \theta, r \cos \theta)$ is 
	increasing. We 
	have
	$$ f_\theta(r) = r\cos \theta \log\left(\frac{\cos \theta}{\sin 
		\theta}\right) 
	+
	(1-r\cos \theta)\log\left(\frac{1-r\cos\theta}{1-r\sin\theta}\right)
	$$
	which can be differentiated along $r$ for a fixed $\theta$ :
	$$f_\theta'(r) = 
	cos\theta\log\left(\frac{1-r\sin\theta}{1-r\cos\theta}\right)+
	\frac{\sin\theta-\cos\theta}{1-r\sin\theta}+
	cos\theta\log\left(\frac{\cos\theta}{\sin\theta}\right).
	$$
	Now, we can differentiate again along $r$ and after some calculations 
	we obtain 
	$$ f_\theta''(r) = \frac{(\sin 
		\theta-\cos\theta)^2}{(1-r\sin\theta)^2(1-r\cos\theta)}>0$$
	which proves that the function $f_\theta$ is convex. It remains to 
	prove that 
	$f_\theta'(0)\geq 0$ for any $\theta\in (0,\pi/2)$. 
	We rewrite $f_\theta'(0)$ as a function of $\theta$ :
	\begin{align*}
	f_\theta'(0) 			
	&=\cos\theta\log\left(\frac{\cos\theta}{\sin\theta}\right)
	+\sin\theta-\cos\theta\\
	& := \phi(\theta)
	\end{align*}
	%		It is easy to check that $\lim\limits_{\theta\rightarrow 
	%			0}\phi(\theta)=+\infty$ and $\lim\limits_{\theta\rightarrow 
	%			\pi/2}\phi(\theta)=1$ but this does not prove the 
	%positivity of 
	%			$\phi$ on $(0,\pi/2)$. 
	Let us assume that there exists $\theta_0\in (0,\pi/2)$ such that 
	$\phi(\theta_0)<0$. Then, in this direction $f_\theta'(0)<0$ and as 
	$f_\theta(0)=0$ 
	for any $\theta\in (0,\pi/2)$, it means that there exists $r_0>0$ such 
	that $f_{\theta_0}(r_0)<0$. Yet, 
	$f_{\theta_0}(r_0)=d(r_0\cos\theta_0,r_0\sin\theta_0)>0$ because of 
	the positivity of the KL divergence. 
	
	So by contradiction, we proved that for all $\theta\in (0,\pi/2)$, 
	$f_\theta'(0)=\phi(\theta)\geq0$ and by convexity $f_\theta$ is 
	non-negative and non-decreasing on $[0,+ \infty)$. 
	
\end{proof}

\if0
\begin{corollary} \label{cor:incr_slopes}
	Let $p,q$ be any fixed real numbers in $(0,1)$, $\alpha >\beta$, then 
	$$ \frac{d(\alpha p,\alpha q)}{\alpha}> \frac{d(\beta p, \beta 
		q)}{\beta}.$$
\end{corollary}
\fi
\subsection{Gaussian payoffs}
The lower bound naturally extends to other classes of distributions, such as Gaussians.
For illustration here we show the lower bound for this case.
We still assume that the means are in $[0,1]$, as before.
We also assume that all payoffs have a common variance $\sigma^2>0$.
Recall that the Kullback-Leibler 
divergence between two distributions with fixed variance $\sigma^2$ is 
$d(p,q)=(p-q)^2/(2\sigma^2)$. Then, the proof of \cref{thm:lower bound} can be repeated with minor differences (in particular, the proof of the analogue of \cref{lem:incrkappa} becomes trivial) and
we get the following result: 
% \todoc{This should be properly stated. But I find it a bit odd that we just add another theorem to the appendix without mentioning it in the main text.}
\begin{theorem} \label{cor:gaussLB}
 For any $(\bar{u}, \bar{v}) \in [0, 1]^K \times [0, 1]^L$ with a unique optimal action
and any uniformly efficient algorithm $\mathcal{A}$ whose regret is $R(n)$,
assuming Gaussian row and column rewards with common variance $\sigma^2$,
	\[\liminf_{n\rightarrow\infty}\frac{R(n)}{\log(n)}\geq 
	\frac{2\sigma^2}{\bar{v}(j^*)}\sum_{i\in[K]\setminus 
	\{i^*\}}\frac{1}{\Delta^\rowvar_i}+
	\frac{2\sigma^2}{\bar{u}(i^*)}\sum_{j\in[L]\setminus 
	\{j^*\}}\frac{1}{\Delta^\colvar_j}\,.
	\]
\end{theorem}

\end{document}